\documentclass[letterpaper]{article} % DO NOT CHANGE THIS
\usepackage{aaai24}  % DO NOT CHANGE THIS
\usepackage{times}  % DO NOT CHANGE THIS
\usepackage{helvet}  % DO NOT CHANGE THIS
\usepackage{courier}  % DO NOT CHANGE THIS
\usepackage[hyphens]{url}  % DO NOT CHANGE THIS
\usepackage{graphicx} % DO NOT CHANGE THIS
\usepackage{natbib}  % DO NOT CHANGE THIS AND DO NOT ADD ANY OPTIONS TO IT
\usepackage{caption} % DO NOT CHANGE THIS AND DO NOT ADD ANY OPTIONS TO IT
\usepackage{algorithm}
\usepackage{algorithmic}
\usepackage{amsmath,amsfonts,amssymb,amsthm}
\usepackage{nicefrac}
\usepackage{microtype}
\usepackage{xcolor}
\usepackage{enumitem}
\usepackage{upgreek}
\usepackage{subcaption}
\usepackage{booktabs}
\usepackage[hang,flushmargin]{footmisc}

\newtheorem{theorem}{Theorem}
\newtheorem{assumption}{Assumption}
\newtheorem{subassumption}{Assumption}
\numberwithin{subassumption}{assumption}

\newtheorem{proposition}[theorem]{Proposition}
\newtheorem{definition}[assumption]{Definition}
\newtheorem{lemma}[theorem]{Lemma}

\newtheorem{remark}{Remark}

\renewcommand{\b}{\boldsymbol{b}}
\newcommand{\C}{\mathcal{C}}
\renewcommand{\d}{\mathrm{d}}
\newcommand{\E}{\mathbb{E}}

\newcommand{\e}{\boldsymbol{e}}
\newcommand{\eps}{\epsilon}

\newcommand{\F}{\mathcal{F}}
\newcommand{\fF}{\mathfrak{F}}

\newcommand{\fG}{\mathfrak{G}}

\renewcommand{\L}{\mathcal{L}}
\renewcommand{\k}{\boldsymbol{k}}

\newcommand{\N}{\mathcal{N}}

\newcommand{\fN}{\mathfrak{N}}

\newcommand{\p}{\boldsymbol{p}}
\renewcommand{\P}{\mathbb{P}}

\newcommand{\qv}[1]{\left\langle #1 \right\rangle}
\newcommand{\R}{\mathbb{R}}
\newcommand{\fR}{\mathfrak{R}}

\newcommand{\bSigma}{\boldsymbol{\Sigma}}

\newcommand{\tr}{\mathrm{tr}}
\newcommand{\TV}{\mathrm{TV}}
\newcommand{\unif}{\mathrm{Unif}}
\newcommand{\var}{\mathrm{Var}}
\newcommand{\w}{\boldsymbol{w}}
\newcommand{\x}{\boldsymbol{x}}
\newcommand{\y}{\boldsymbol{y}}
\newcommand{\Z}{\mathbb{Z}}

\DeclareMathOperator*{\argmin}{arg\,min}

\title{Statistical Spatially Inhomogeneous Diffusion Inference}
\author{ Yinuo Ren\textsuperscript{\rm 1}, Yiping Lu\textsuperscript{\rm 2},
    Lexing Ying\textsuperscript{\rm 1,3}, Grant M. Rotskoff\textsuperscript{\rm
    1,4}\footnote{\noindent GMR acknowledges support from a Google Research Scholar Award.} } \affiliations { \textsuperscript{\rm 1}Insitute for Computational and
    Mathematical Engineering (ICME), Stanford University\\
    \textsuperscript{\rm 2}Courant Institute of Mathematical Sciences, New York
    University\\
    \textsuperscript{\rm 3}Department of Mathematics, Stanford University\\
    \textsuperscript{\rm 4}Department of Chemistry, Stanford University\\
    \{yinuoren, lexing, rotskoff\}@stanford.edu, yiping.lu@nyu.edu }

\begin{document}

\maketitle

\begin{abstract}
    Inferring a diffusion equation from discretely-observed measurements is a
    statistical challenge of significant importance in a variety of fields, from
    single-molecule tracking in biophysical systems to modeling financial
    instruments. Assuming that the underlying dynamical process obeys a
    $d$-dimensional stochastic differential equation of the form
        $$
            \d\x_t=\b(\x_t)\d t+\Sigma(\x_t)\d\w_t,
        $$
     we propose neural network-based estimators of both the drift $\b$ and the
     spatially-inhomogeneous diffusion tensor $D = \Sigma\Sigma^{T}$ and provide
     statistical convergence guarantees when $\b$ and $D$ are $s$-H\"older
     continuous. Notably, our bound aligns with the minimax optimal rate
     $N^{-\frac{2s}{2s+d}}$ for nonparametric function estimation even in the
     presence of correlation within observational data, which necessitates
     careful handling when establishing fast-rate generalization bounds. Our
     theoretical results are bolstered by numerical experiments demonstrating
     accurate inference of spatially-inhomogeneous diffusion tensors.
\end{abstract}

\section{Introduction}

The dynamical evolution of a wide variety of natural processes, from molecular
motion within cells to atmospheric systems, involves an interplay between
deterministic forces and noise from the surrounding environment. While it is
possible to observe time series data from such systems, in general the
underlying equation of motion is not known analytically. Stochastic differential
equations offer a powerful and versatile framework for modeling these complex
systems, but inferring the deterministic drift and diffusion tensor from time
series data remains challenging, especially in high-dimensional settings. Among
the many strategies
proposed~\cite{crommelin2011diffusion,frishman2020learning,nickl2022inference},
there are few rigorous results on the optimality and convergence properties of
estimators of, in particular, spatially-inhomogeneous diffusion tensors.

Many numerical algorithms have been proposed to infer the drift and diffusion,
accommodating various settings, including
one-dimensional~\cite{sura2002note,papaspiliopoulos2012nonparametric,davis2022estimation}
and multidimensional
SDEs~\cite{pokern2009remarks,frishman2020learning,crommelin2011diffusion}. Also,
the statistical convergence rate has been extensively studied for both the
one-dimensional
case~\cite{dalalyan2005sharp,dalalyan2006asymptotic,pokern2013posterior,aeckerle2018sup}
and the multidimensional
cases~\cite{van2006convergence,dalalyan2007asymptotic,van2016gaussian,nickl2017nonparametric,nickl2020nonparametric,oga2021drift,nickl2022inference}.
For parametric estimators using a Fourier or wavelet basis, the statistical
limits of estimating the spatially-inhomogeneous diffusion tensor have been
rigorously characterized~\cite{hoffmann1997minimax,hoffmann1999adaptive}.
However, strategies based on such decompositions do not scale to
high-dimensional problems, which has motivated the investigation of neural
networks as a more flexible representation of the SDE coefficients
\cite{han2018solving,rotskoff2022active,khoo2021solving,li2021markov}. 

Thus, we consider the nonparametric neural network
estimator~\cite{suzuki2018adaptivity,oono2019approximation,schmidt2020nonparametric}
as our ansatz function class, which has achieved great success in estimating SDE
coefficients
empirically~\cite{xie2007estimation,zhang2018deep,han2018solving,wang2022neural,lin2023computing}.
We aim to build statistical guarantees for such neural network-based estimators.
The most related concurrent work is~\cite{gu2023stationary}, where the authors
provide a convergence guarantee for the neural network estimation of the drift
vector and the homogeneous diffusion tensor of an SDE by solving appropriate
supervised learning tasks. However, their approach assumes that the data
observed along the trajectory are independently and identically distributed from
the stationary distribution. Additionally, the generalization bound used
in~\cite{gu2023stationary} is not the fast rate generalization
bound~\cite{bartlett2005local,koltchinskii2006local}, resulting in a suboptimal
final guarantee. Therefore, we seek to bridge the gap between the i.i.d. setting
and the non-i.i.d. ergodic setting using mixing conditions and extend the
algorithm and analysis to the spatially-inhomogeneous diffusion estimation. We
show that neural estimators have the ability to achieve standard minimax optimal
nonparametric function estimation rates even when the data are non-i.i.d. 

\subsection{Contribution}

In this paper, we construct a fast-rate error bound for estimating a
multi-dimensional spatially-inhomogeneous diffusion process based on non-i.i.d
ergodic data along a single trajectory. Our contributions are as follows:
\begin{itemize}
\item We derive for neural network-based diffusion estimators a convergence rate
that matches the minimax optimal nonparametric function estimation rate for the
$s$-H\"older continuous function class~\cite{tsybakov2009introduction};
\item Our analysis explores the $\upbeta$-mixing condition to address the
correlation present among observed data along the trajectory, making our result
readily applicable to a wide range of ergodic diffusion processes;
\item We present numerical experiments, providing empirical support for our
derived convergence rate and facilitating further applications of neural
diffusion estimators in various contexts with theoretical assurance.
\end{itemize}
Our theoretical bound depicts the relationships between the error of
nonparametric regression, numerical discretization, and ergodic approximation,
and provides a general guideline for designing data-efficient, scale-minimal,
and statistically-optimal neural estimators for diffusion inference.

\subsection{Related Works}

\paragraph{Inference of diffusion processes from data} 

The problem of inferring the drift and diffusion coefficients of an SDE from
data has been studied extensively in the literature. The setting with access to
the whole continuous trajectory is studied
by~\cite{dalalyan2006asymptotic,dalalyan2007asymptotic,strauch2015sharp,strauch2016exact,nickl2020nonparametric,rotskoff2019dynamical},
in which the diffusion tensor can be exactly identified using quadratic
variation arguments, and thus only the drift inference is considered. Many works
focus on the numerical recovery of both the drift vector and the diffusion
tensor in the more realistic setting when only discrete observations are
available, including methods based on local
linearization~\cite{ozaki1992bridge,shoji1998estimation}, martingale estimating
functions~\cite{bibby1995martingale}, maximum likelihood
estimation~\cite{pedersen1995consistency,ait2002maximum}, and Markov chain Monte
Carlo~\cite{elerian2001likelihood}. We refer readers
to~\cite{sorensen2004parametric,lopez2021parametric} for an overview of
parametric approaches. A spectral method that estimates the eigenpairs of the
Markov semigroup operator is proposed in~\cite{crommelin2011diffusion}, and a
nonparametric Bayesian inference scheme based on the finite element method is
studied in~\cite{papaspiliopoulos2012nonparametric}. As for the statistical
convergence rate of the drift and diffusion inference, a line of pioneering
works is by~\cite{hoffmann1997minimax,hoffmann1999adaptive,hoffmann1999lp},
where the minimax convergence rate of the one-dimensional diffusion process is
derived for Besov spaces and matched by adaptive wavelet estimators. Alternative
analyses mainly follow a Bayesian methodology, with notable results
by~\cite{nickl2017nonparametric,nickl2020nonparametric,nickl2022inference} in
both high- and low-frequency schemes. 

\paragraph{Solving high-dimensional PDEs with deep neural networks} The curse of
dimensionality has stymied efforts to solve high-dimensional partial
differential equations (PDEs) numerically. However, deep learning has
demonstrated remarkable flexibility and adaptivity in approximating
high-dimensional functions, which indeed has led to significant advances in
computer vision and natural language processing. Recently, a series of works
\cite{han2018solving,yu2018deep,karniadakis2021physics,khoo2021solving,long2018pde,zang2020weak,kovachki2021neural,lu2019deeponet,li2022semigroup,rotskoff2022active}
have explored solving PDEs with deep neural networks, achieving impressive
results for a diverse collection of tasks. These approaches rely on representing
PDE solutions using neural networks, and various schemes propose different loss
functions to obtain a solution. For instance, \cite{han2018solving} uses the
Feynman-Kac formulation to convert PDE solving into a stochastic control
problem, \cite{karniadakis2021physics} solves the PDE minimizing the strong
form, while \cite{zang2020weak} solves weak formulations of PDEs with an
adversarial approach.

\paragraph{Theoretical guarantees for neural network-based PDE solvers}
Statistical learning theory offers a powerful toolkit to prove theoretical
convergence results for PDE solvers based on deep learning. For example,
\cite{weinan2022some,chen2021representation,marwah2021parametric,marwah2022neural}
investigated the regularity of PDEs approximated by neural networks and
\cite{nickl2020convergence,duan2021convergence,lu2021machine,hutter2021minimax,lu2022sobolev}
consider the statistical convergence rate of various machine learning-based PDE
solvers. However, most of these optimality results are based on concentration
results that assume the sampled data are independent and identically
distributed. This i.i.d. assumption is often violated in various financial and
biophysical applications, for example, time series prediction, complex system
analysis, and signal processing. Among many possible relaxations to this i.i.d.
setting, the scenario, where data are drawn from a strong mixing process, has
been widely adopted~\cite{bradley2005basic}. Inspired by the first work of this
kind~\cite{yu1994rates}, many authors exploited a set of mixing concepts such as
$\alpha$-mixing~\cite{zhang2004statistical,steinwart2009fast}, $\beta$- and
$\phi$-mixing~\cite{mohri2008rademacher,mohri2010stability,kuznetsov2017generalization,ziemann2022single},
and $\C$-mixing~\cite{hang2017bernstein}. We refer readers
to~\cite{hang2016learning} for an overview of this line of research.

\paragraph{Notations}

We will use $\lesssim$ and $\gtrsim$ to denote the inequality up to a constant
factor and $\asymp$ the equality up to a constant factor. 

\begin{definition}[H\"{o}lder space]
	We denote the H\"{o}lder space of order $s\in\R$ with constant $M>0$ by
	$\C^s(\R^d, M)$, \emph{i.e.}
	$$
		\begin{aligned}
            &\C^s(\R^d, M) = 
            \bigg\{ f:\R^d\to\R \bigg|\\
            &\sum_{|\boldsymbol{\alpha}|<s}\left\|\partial^{\boldsymbol{\alpha}} f\right\|_{\infty}+\sum_{|\boldsymbol{\alpha}|=\lfloor s\rfloor} \sup _{\x \neq \y}
		\frac{\left|\partial^{\boldsymbol{\alpha}} f(\x)-\partial^{\boldsymbol{\alpha}} f(\y)\right|}{|\x-\y|^{s-\lfloor s\rfloor}} < M \bigg\}.
        \end{aligned}
	$$
\end{definition}

\section{Problem Setting}
\label{sec:probset}

Suppose we have access to a sequence of $N$ discrete position snapshots
$(\x_{k\tau})_{k=0}^N$ along a single trajectory $(\x_t)_{0\leq t\leq T}$, where
the time step $\tau = T/N$ and $(\x_t)_{t\geq 0}$ is the solution to the
following It\^o stochastic differential equation:
\begin{equation}
	\d \x_t = \b(\x_t) \d t + \Sigma(\x_t) \d \w_t,
	\label{eq:sde}
\end{equation}
where $\b :\R^\d\to \R^d$, $\Sigma :\R^d\to \R^{d\times r}$, and $(\w_t)_{t\geq
0}$ is an $r$-dimensional Wiener process. We refer the vector field $\b(\x)$ as
the drift vector, and define the diffusion tensor as $D(\cdot) =
\tfrac12\Sigma(\cdot)\Sigma(\cdot)^\top$. As noted in~\cite{lau2007state}, any
interpolation between the It\^o convention and other conventions for stochastic
calculus can be transformed into the It\^o convention by an additional term to
the drift vector, and therefore, we work with the It\^o convention throughout
this paper\footnote{The It\^o convention along with others represent different
methods to extend the Riemann integral to stochastic processes. Roughly
speaking, Ito uses the left endpoint of the interval for functional value in the
Riemann sum. We adopt the It\^o convention due to several martingale properties
it introduces which are mathematically convenient for statements and proofs.}.

\begin{remark}
    Our focus on the inhomogeneity in the space variable stems from the fact
    that when the SDE coefficients are time-dependent, it becomes very
    challenging to infer them from a singular observational trajectory,
    \emph{i.e.} with only one observation at each time point and we would leave
    this case with multiple trajectories for future work.
\end{remark}

For simplicity, we will be working on $\Omega = [0,1)^d$ with periodic
boundaries, \emph{i.e.} the $d$-dimensional torus $\widetilde{\Omega} =
\R^d/\Z^d$. Points on the torus $\widetilde{\Omega}$ are represented by
$\widetilde{\x}$, where $\widetilde{\cdot}$ denotes the canonical map and
$\x\in\R^d$ is a representative of the equivalence class $\widetilde{\x}$. The
Borel $\sigma$-algebra on $\widetilde{\Omega}$ coincides with the sub-$\sigma$
algebra of 1-periodic Borel sets of $\R^d$. We refer readers
to~\cite{papanicolau1978asymptotic} for further mathematical details of
homogenization with tori.  We further assume the drift and diffusion
coefficients in~\eqref{eq:sde} satisfy the following regularity assumptions:
\stepcounter{assumption}
\begin{subassumption}[Periodicity]
	$\b(\x)$, $\Sigma(\x)$, and $D(\x)$ are 1-periodic for all variables.
	\label{ass:per}
\end{subassumption}
\begin{remark}
    This assumption is primarily for simplicity, and has been adopted in many
    previous works on the statistical inference of SDE coefficients,
    \emph{e.g.}~\cite{nickl2020nonparametric}. This allows us to bypass the
    technicalities concerning boundary conditions, which might detract from our
    main contributions. 
\end{remark}
\begin{subassumption}[H\"{o}lder-smoothness]
	Each entry $b_i(\x),\Sigma_{ij}(\x), D_{ij}(\x)\in\C^{s}(\R^d, M)$ for some
	$s\geq 2$ and $M>0$.
	\label{ass:holder}
\end{subassumption}
\begin{subassumption}[Uniform ellipticity]
	It holds that $r\geq d$ and there exists a constant $c>0$ such that
	$D(\x)\succ c I$, \emph{i.e.} $\sum_{i,j=1}^d D_{ij}(\x)\xi_i\xi_j\geq c
	\|\xi\|^2$ for any $\xi\in\R^d$, holds uniformly for any $\x\in \R^d$.
	\label{ass:reg}
\end{subassumption}
\begin{remark}
    This uniform ellipticity is commonly assumed across the analysis of the
    Fokker-Planck equation. It guarantees the Fokker-Planck equation has a
    unique strong solution with regularity properties that are essential for the
    analysis of asymptotic behavior and numerical approximation of the solution.
    We refer readers to~\cite{stroock1997multidimensional,bogachev2022fokker}
    for more detailed discussions.
\end{remark}

Since $\b(\x)$ and $\Sigma(\x)$ are 1-periodic, the process $(\x_t)_{t\geq 0}$
in~\eqref{eq:sde} can thus be viewed as a process $(\widetilde{\x}_t)_{t\geq
0}:=(\widetilde{\x_t})_{t\geq 0}$ on the torus $\widetilde{\Omega}$. Denote the
transition kernel of the process $\x_t$ by $\P^t(\x,\cdot):=\P(\x_t \in \cdot|
\x_0 = \x)$, the transition kernel of the corresponding process
$\widetilde{\x}_t$ satisfies:
\begin{equation}
    \widetilde{\P}^t(\widetilde{\x},\cdot) = \sum_{\k = (k_1,\cdots,k_d)\in\Z^d}\P^t\left(\x,\cdot+\sum_{i=1}^d k_i \e_i\right),
    \label{eq:widetildeP}
\end{equation}
where $\e_i$ is the $i$-th standard basis vector in $\R^d$. When no confusion
arises, we will use $\widetilde{\x}$ to denote its representative in the
fundamental domain $\Omega$ in the following. 

\section{Spatially Inhomogeneous Diffusion Estimator}

In this section, we aim to build neural estimators of both the drift and
diffusion coefficients based on a sequence of $N$ discrete observations
$\left(\x_{k\tau}\right)_{k=0}^N$ along a single trajectory of the
SDE~\eqref{eq:sde}. A straightforward neural drift estimator allows us to
subsequently construct a simple neural estimator of the diffusion tensor. In
what follows, we introduce and prove the convergence of these neural estimators.
Without loss of generality, we assume $\tau\leq 1$ and $T\geq 1$, and denote the
$\sigma$-algebra generated by all possible sequences
$\left(\x_{k\tau}\right)_{k=0}^N$ as $\F_{\tau,N}(\b, D)$.

\subsection{Neural Estimators}
\label{sec:est}

We define $\L_T^{\b}(\hat \b)$ and $\L_T^{D}(\hat{D})$ as the objective function
for drift and diffusion estimation, respectively, by noticing that the ground
truth drift vector $\b$ can be represented as the minimizer of the following
objective function as the time step $\tau\to 0$ and the time horizon $T\to
\infty$:
\begin{equation}
    \L_T^{\b}(\hat \b; (\x_t)_{0\leq t \leq T}):=\dfrac{1}{T}\int_{0}^T\left\| {\hat\b}(\x_t) -  \dfrac{1}{\tau}\Delta \x_t \right\|_2^2\d t,
\label{eq:driftpop}
\end{equation}
where $\Delta \x_t = \x_{t+\tau}-\x_t$. With the ground truth drift vector $\b$,
the ground truth diffusion tensor can also be represented as the minimizer of
the following objective function as $\tau\to 0$ and $T\to \infty$: 
\begin{equation}
    \begin{aligned}
        &\L_{T}^D(\hat{D}; (\x_t)_{0\leq t \leq T}, \b):=\\&\dfrac{1}{T}\int_{0}^T \bigg\| \hat{D}(\x_t)- \dfrac{\left(\Delta\x_{t}-\b(\x_t)\tau\right)\left(\Delta\x_{t}-\b(\x_t)\tau\right)^\top}{2\tau}\bigg\|_F^2\d t,
    \end{aligned}
    \label{eq:diffusionpop}
\end{equation}
where $\|\cdot\|_F$ is the Frobenius norm of a matrix.

Based on the discussions in the last section, we will only estimate the value of
$\hat \b$ and $\hat D$ in the fundamental domain $\Omega$ and then extend it to
the whole space by periodicity. Therefore, using our data $(\x_{k\tau})_{k=0}^N$
as quadrature points, we approximate the objective function for drift
estimation~\eqref{eq:driftpop} as:
\begin{equation}
    \tilde{\L}_N^{\b}(\hat \b; (\x_{k\tau})_{k=0}^N):=\dfrac{1}{N}\sum_{k=0}^{N-1} \left\| \dfrac{\x_{(k+1)\tau}-\x_{k\tau}}{\tau} -\hat{\b}(\widetilde{\x}_{k\tau}) \right\|_2^2,
    \label{eq:driftemploss}
\end{equation}
and the objective function for diffusion estimation~\eqref{eq:diffusionpop} as
\begin{equation}
    \begin{aligned}
        &\tilde{\L}_N^{D}(\hat{D};(\x_{k\tau})_{k=0}^N, \b):=\dfrac{1}{N}\sum_{k=0}^{N-1} \\
        &\left\| \dfrac{\left(\Delta\x_{k\tau}-{\b}(\widetilde{\x}_{k\tau})\tau\right)\left(\Delta\x_{k\tau}-{\b}(\widetilde{\x}_{k\tau})\tau\right)^\top}{2\tau} -\hat{D}(\widetilde{\x}_{k\tau}) \right\|_F^2.
    \end{aligned}
    \label{eq:diffusionemploss}
\end{equation}
We will refer to $\tilde{\L}_N^{\b}(\hat \b; (\x_{k\tau})_{k=0}^N)$ and
$\tilde{\L}_N^{D}(\hat{D};(\x_{k\tau})_{k=0}^N, \b)$ as the estimated empirical
loss for drift and diffusion estimation, respectively.

\begin{algorithm}[!htb]
\caption{Diffusion inference within function class $\fG$}
\begin{algorithmic}[1]
\STATE Find the drift estimator $$\hat \b:= \argmin_{\bar \b \in \fG^d}
\tilde{\L}^{\b}_N(\bar \b; (\x_{k\tau})_{k=0}^N);$$ \STATE Find the diffusion
estimator $$\hat D:= \argmin_{\bar D \in \fG^{d\times d}}
\tilde{\L}^D_N(\bar{D};(\x_{k\tau})_{k=0}^N, \hat \b),$$where $\hat \b$ is the
drift estimator obtained in the first step as an approximation for the ground
truth $\b$.
\end{algorithmic}
\label{alg:alg}
\end{algorithm}

We then parametrize the drift vector and the diffusion tensor within a
\emph{hypothesis function class} $\fG$ and solve for the estimators by
optimizing the corresponding estimated empirical loss, as in
Algorithm~\ref{alg:alg}. Following foundational works
including~\cite{oono2019approximation,schmidt2020nonparametric,chen2022nonparametric},
we adopt sparse neural networks $\fN(L,\p,S,M)$ as our hypothesis function class
$\fG$, which is defined as follows. A neural network with depth $L$ and width
vector $\p = (p_0 , \cdots, p_{L+1} )$ has the following form $f:\R^{p_0}
\to\R^{p_{L+1}}$ with
\begin{equation}
		\x \mapsto f(\x) = W_L(\sigma(W_{L-1}(\cdots\sigma(W_0 \x - \w_1)\cdots) - \w_{L})),
    \label{eq:nn}
\end{equation}
where $W_i\in \R^{p_{i+1}\times p_i}$ are the weight matrices, $\w_i\in\R^{p_i}$
are the shift vectors, and $\sigma(\cdot)$ is the element-wise ReLU activation
function. We also bound all parameters in the neural network by unity as
in~\cite{schmidt2020nonparametric,suzuki2018adaptivity}.
\begin{definition}[Sparse neural network]
	Let $\fN(L,\p,S,M)$ be the function class of ReLU-activated neural networks
	with depth $L$ and width $\p$ that has at most $S$ non-zero entries with the
	function value uniformly bounded by $M$ and all parameters bounded by 1,
	\emph{i.e.}
	$$
		\begin{aligned}
			&\fN(L,\p,S,M)  = \bigg\{ f(\x) \text{ has the form of~\eqref{eq:nn}}\bigg| \\
			&\quad \quad \quad \sum_{i=0}^L \|W_i\|_0 + \sum_{i=1}^L \|\w_i\|_0\leq S,\|f\|_\infty\leq M\\
            &\quad \quad \quad \quad \max_{i=0,\cdots,L}\|W_i\|_\infty \vee \max_{i=1,\cdots,L}\|\w_i\|_\infty\leq 1\bigg\},
		\end{aligned}
	$$
	where $\|\cdot\|_0$ is the number of non-zero entries of a matrix (or a
	vector) and $\|\cdot\|_\infty$ is the maximum absolute value of a matrix (or
	a vector).
\end{definition}
Since we are using the neural network for nonparametric estimation in
$\Omega\subset\R^d$, we will assume $p_0 = d$ and $p_{L+1} = 1$ in the following
discussion. 

\subsection{Ergodicity}

Optimal convergence rates of neural network-based PDE solvers, as showcased
in~\cite{nickl2020convergence,lu2021machine,gu2023stationary}, are typically
established under the assumption of data independence. However, the presence of
time correlations in the observational data $\left(\x_{k\tau}\right)_{k=0}^N$
from a single trajectory significantly complicates the task of setting an upper
bound for the convergence of the neural estimators obtained by
Algorithm~\ref{alg:alg}. In this context, we fully explore the ergodicity of the
diffusion process, bound the ergodic approximation error by the $\upbeta$-mixing
coefficient, and show that the exponential ergodicity condition, which is
naturally satisfied by a wide range of diffusion processes, is sufficient for
the fast rate convergence of the proposed neural estimators.

We first introduce the definition of exponential ergodicity:
\begin{definition}[Exponential ergodicity~\cite{down1995exponential}]
	A diffusion process $(X_t)_{t\geq 0}$ with domain $\Omega$ is uniformly
	\emph{exponential ergodic} if there exists a unique stationary distribution
	$\mu$ that for any $x\in \Omega$,
	$$
		\|\P^t(x,\cdot) - \mu \|_{\TV}\leq M_{\mu}(x)\exp(-C_{\mu} t),
	$$
    where $M_{\mu}(x), C_{\mu}>0$.
\end{definition}

As a direct consequence of~\cite[Theorem 3.2]{papanicolau1978asymptotic} and the
compactness of the torus $\widetilde \Omega$, we have the following result:
\begin{proposition}[Exponential ergodicity of $(\widetilde{\x}_t)_{t\geq 0}$]
    The diffusion process $(\widetilde{\x}_t)_{t\geq 0}$, the image of
    $(\x_t)_{t\geq 0}$ in~\eqref{eq:sde} under the quotient map, is uniformly
    exponential ergodic with respect to a unique stationary distribution
    $\widetilde{\Pi}$ on the torus $\widetilde{\Omega}$ under
    Assumptions~\ref{ass:per}, \ref{ass:holder}, and \ref{ass:reg}. Especially,
    there exist constants $M_{\widetilde \Pi}, C_{\widetilde \Pi}>0$ that only
    depend on $c$, $\b$, and $D$, such that for any $\widetilde\x\in \Omega$,
	$$
		\|\widetilde \P^t(\widetilde\x,\cdot) - \widetilde\Pi \|_{\TV}\leq M_{\widetilde \Pi}\exp(-C_{\widetilde \Pi} t).
	$$
\label{prop:ergodic}
\end{proposition}
See~\cite{kulik2017ergodic} for further discussions and required regularities
for this property beyond the torus setting.

The ergodicity of stochastic processes is closely related to the notion of
\emph{mixing conditions}, which quantifies the  ``asymptotic independence'' of
random sequences. One of the most utilized mixing conditions for stochastic
processes is the following $\upbeta$-mixing condition:
\begin{definition}[$\upbeta$-mixing condition~\cite{kuznetsov2017generalization}]
	The \emph{$\upbeta$-mixing coefficient} of a stochastic process
	$(X_t)_{t\geq 0}$ with respect to a probability measure $\mu$ is defined as
	$$ \upbeta(t; (X_t)_{t\geq 0}, \mu) := \sup_{s\geq
	0}\E_{\F_0^t}\left[\|\mu-\P_{t+s}^\infty(\cdot|\F_0^t)\|_{\TV}\right],$$
	where $\F_a^b$ is the $\sigma$-algebra generated by $(X_t)_{a\leq t\leq b}$,
	and $\P_a^b$ is the law of $(X_t)_{a\leq t\leq b}$. Especially, when $
	\upbeta(t; (X_t)_{t\geq 0}, \Pi) \leq M_\upbeta \exp(-C_\upbeta t)$ for some
	constants $M_\upbeta, C_\upbeta>0$, we say $X_t$ is \emph{geometrically
	$\upbeta$-mixing} with respect to $\mu$.
\end{definition}

By taking $\mu$ as the stationary distribution $\widetilde \Pi$ in the
definition above, the proposition follows:
\begin{proposition}[$\upbeta$-mixing condition of $(\widetilde{\x}_t)_{t\geq 0}$]
	$\upbeta(t;(\widetilde{\x}_t)_{t\geq 0}, \widetilde{\Pi})\leq
	M_{\widetilde{\Pi}}\exp(-C_{\widetilde{\Pi}} t)$, \emph{i.e.}
	$\widetilde{\x}_t$ is geometrically $\upbeta$-mixing with respect to
	$\widetilde{\Pi}$.
	\label{prop:beta}
\end{proposition}

We will denote the pushforward of the invariant measure $\widetilde{\Pi}$ under
the following inverse of the canonical map
$\iota^{-1}:\widetilde{\Omega}\to\Omega$ also as $\widetilde\Pi$.

\subsection{Convergence Guarantee}
\label{sec:convergence}

In this section, we describe the main upper bound for the neural estimators in
Algorithm~\ref{alg:alg}. We also present a theoretical guarantee for drift and
diffusion estimation in Theorem~\ref{thm:drift} and~\ref{thm:diffusion},
respectively. Our main result shows that estimating the drift and diffusion
tensor can achieve the standard minimax optimal nonparametric function
estimation convergence rate, even with non-i.i.d. data. 

Due to the ergodic theorem~\cite[Theorem 5.3.3]{kulik2017ergodic} under the
exponential ergodicity condition and the property of It\^o process, the
\emph{bias} part of the objective functions $\L_T^{\b}(\hat \b; (\x_t)_{0\leq t
\leq T})$ and $\L_{T}^D(\hat{D}; (\x_t)_{0\leq t \leq T}, \b)$ for drift and
diffusion estimation as defined in~\eqref{eq:driftpop}
and~\eqref{eq:diffusionpop} converge to 
\begin{equation}
    \begin{aligned}
        \mathcal{L}^{\b}_{\widetilde\Pi}(\hat \b)&:=\mathbb{E}_{\widetilde\x\sim\widetilde\Pi}\left[\|\hat \b(\widetilde\x)-\b(\widetilde\x)\|_2^2\right]\\
        \text{and}\quad \mathcal{L}^D_{\widetilde\Pi}(\hat D)&:=\mathbb{E}_{\widetilde\x\sim\widetilde\Pi}\left[\|\hat D(\widetilde\x)-D(\widetilde\x)\|^2_F\right],
    \end{aligned}
    \label{eq:poploss}
\end{equation} 
as $\tau\to 0$ and $T\to \infty$, which we will refer to as the \emph{population
loss} for drift and diffusion estimation, respectively. Our convergence
guarantee is thus built on these population losses.

\begin{theorem}[Upper bound for drift estimation in $\fN(L,\p,S,M)$]
	Suppose the drift vector $\b\in\C^s(\Omega, M)$, and the hypothesis class
	$\fG = \fN(L,\p,S,M)$ with $$K\asymp T^{\frac{d}{2s+d}},\quad L \lesssim
	\log K,\quad \|\p\|_\infty\lesssim K,\quad S\lesssim K\log K.$$ Then with
	high probability the minimizer $\hat{\b}$ obtained by
	Algorithm~\ref{alg:alg} satisfies
	$$
		\E_{(\x_{k\tau})_{k=0}^N\sim\F_{\tau,N}(\b, D)}\left[\L^b_{\widetilde\Pi}(\hat{\b})\right]\lesssim  T^{-\frac{2s}{2s+d}}\log^3 T+\tau.
	$$
	\label{thm:drift}
\end{theorem}

\begin{theorem}[Upper bound for diffusion estimation in $\fN(L,\p,S,M)$]
	Suppose the diffusion tensor $D\in\C^s(\Omega, M)$, and the hypothesis class
	$\fG = \fN(L,\p,S,M)$ with $$K\asymp N^{\frac{d}{2s+d}},\quad L \lesssim
	\log K,\quad \|\p\|_\infty\lesssim K,\quad S\lesssim K\log K.$$ Then with
	high probability the minimizer $\hat{D}$ obtained by Algorithm~\ref{alg:alg}
	satisfies
	\begin{equation}
	    \E_{(\x_{k\tau})_{k=0}^N\sim\F_{\tau,N}(\b, D)}\left[\L^D_{\widetilde\Pi}(\hat{D})\right]\lesssim  N^{-\frac{2s}{2s+d}}\log^3 N+\tau+ \dfrac{\log^2 N}{T}.
     \label{eq:difrate}
	\end{equation}
	\label{thm:diffusion}
\end{theorem}

\begin{remark}
    In this remark, we explain the meaning of each term in the convergence
    rate~\eqref{eq:difrate}: 
    \begin{itemize}
        \item The term $N^{-\frac{2s}{2s+d}}\log^3 N$ matches the standard
        minimax optimal rate $N^{-\frac{2s}{2s+d}}$ up to an extra
        $\mathrm{poly}(\log n)$ factor. This is characteristic of performing
        nonparametric regression for $s$-H\"older continuous functions with $N$
        noisy observations~\cite{tsybakov2009introduction}.  This is different
        from the drift estimation (Theorem \ref{thm:drift}) in which the
        nonparametric dependency is on $T$ instead of $N$ with a further
        $\frac{\log^2 N}{T}$ term which is discussed below. 
        \item The term $\tau$ represents a bias term that arises due to the
        finite resolution of the observations $(\x_{k\tau})_{t=0}^N$.
        Specifically, this term encapsulates the error incurred while
        approximating the objective function $\L_{T}^D(\hat{D}; (\x_t)_{0\leq t
        \leq T}, \b)$ by the estimated empirical loss
        $\tilde{\L}_N^{D}(\hat{D};(\x_{k\tau})_{k=0}^N, \hat \b)$ with numerical
        quadrature and finite difference computations;
        \item The term $\frac{\log^2 N}{T}$ quantifies the error in
        approximating the population loss $\mathcal{L}^D_{\widetilde\Pi}(\hat
        D)$ by the objective function $\L_{T}^D(\hat{D}; (\x_t)_{0\leq t \leq
        T}, \b)$ by applying the ergodic theorem up to time horizon $T$. This
        term essentially signifies the portion of the domain that the trajectory
        has not yet traversed.
        Refs.~\cite{hoffmann1997minimax,hoffmann1999adaptive} only provide
        guarantee for $\L_{T}^D(\hat{D}; (\x_t)_{0\leq t \leq T}, \b)$ and thus
        this term is not included.
    \end{itemize}
\end{remark}

\subsection{Proof Sketch}
\label{sec:proof}

In this section, we omit the dependency of the losses on the data
$(\x_{k\tau})_{k=0}^N$ for notational simplicity.

To obtain a unified proving approach for both drift and diffusion estimation, it
is useful to think of our neural estimator as a function regressor with
imperfect supervision signals. We consider an estimator $\hat{g}\in\fG$ of an
arbitrary function $g^0$ as the ground truth obtained by minimizing over the
estimated empirical loss
\begin{equation}
	\tilde{\L}_N^{g^0}(\hat{g}) = \dfrac{1}{N}\sum_{k=0}^{N-1} \left( g^0(\widetilde\x_{k\tau})+\Delta Z_{k\tau} -\hat{g}(\widetilde\x_{k\tau}) \right)^2,
	\label{eq:estemploss}
\end{equation}
where the supervision signal is polluted by the noise given by $\Delta Z_{k\tau}
= Z_{(k+1)\tau}-Z_{k\tau}$, with $Z_t$ being an $\F_t$-adapted continuous
semimartingale. Following Doob's decomposition, we write $Z_{t} = A_{t}+M_{t}$,
where $(A_t)_{t\geq 0}$ is a continuous process with finite variation and is
deterministic on $[k\tau,(k+1)\tau]$ conditioned on $\F_0^{k\tau}$ as
$$A_{t}=\sum_{k=0}^{N-1}\left( \E\left[ Z_{t\wedge (k+1)\tau}\mid \F_{t\wedge
k\tau} \right] - Z_{t\wedge k\tau} \right)$$ and $(M_t)_{t\geq 0}$ forms a local
martingale as $$M_t = \sum_{k=0}^{N-1}\left( Z_{t\wedge (k+1)\tau} - \E\left[
Z_{t\wedge (k+1)\tau}\mid \F_{t\wedge k\tau} \right] \right).$$ The population
loss $\mathcal{L}^{g^0}_{\widetilde\Pi}(\hat g)$ can also be similarly defined
as in~\eqref{eq:poploss}. Additionally, we define the \emph{empirical loss} for
the estimator $\hat g$ as $$\hat{\L}_N^{g^0}(\hat g):=
\dfrac{1}{N}\sum_{k=0}^{N-1} \left( g^0(\widetilde\x_{k\tau})
-\hat{g}(\widetilde\x_{k\tau}) \right)^2.$$

In our proof, we first show that as long as the following two conditions hold
for the noise $(\Delta Z_{k\tau})_{k= 0}^N$, the minimax optimal nonparametric
function estimation rate would be achieved:
\begin{assumption}
	For any $k$, the continuous finite variation process $(A_t)_{t\geq 0}$
	satisfies $$\E\left[\dfrac{1}{N}\sum_{k=0}^{N-1} (\Delta A_{k\tau})^2
	\right]\leq C_A\tau.$$
	\label{ass:bias}
\end{assumption}
\begin{assumption}
	For some $\gamma\leq 1$, the local martingale $(M_t)_{t\geq 0}$ satisfies
	$$\max_k\left|\E\left[\Delta\qv{M}_{ k\tau} \big|
	\F_0^{k\tau}\right]\right|\leq C_M\tau^{-\gamma},$$ where $\qv{\cdot}$
	denotes the quadratic variation.
	\label{ass:mart}
\end{assumption}
\begin{remark}
Based on the noise decomposition $\Delta Z_{k\tau} = \Delta A_{k\tau} + \Delta
M_{k\tau}$, the term $\Delta A_{k\tau}$ can be intuitively understood as the
bias of the data. This bias is caused by the numerical scheme employed for
computing $g^0_{k\tau}$. On the other hand, the term $\Delta M_{k\tau}$
represents the martingale noise added to the data, which can be considered
analogous to the i.i.d. noise in the common nonparametric estimation settings.
Assumption~\ref{ass:bias} essentially implies that the estimator $\hat{g}$ is
consistent, for its expectation converges to $g^0$ as $\tau \to 0$. Meanwhile,
Assumption~\ref{ass:mart} assumes that the variance of the noise present in the
data is at most of order $\mathcal{O}(\tau^{-1})$.
	\label{rem:noise}
\end{remark}

To overcome the correlation of the observed data, we adopt the following
\emph{sub-sampling} technique as
in~\cite{yu1994rates,mohri2008rademacher,hang2017bernstein}: For a sufficiently
large $l\geq 1$ such that $N = n l$\footnote{Here we assume $N$ is divisible by
$l$ without loss of generality.}, we split the original $N$ correlated samples
$S^N := (\x_{k\tau})_{k=0}^N$ into $l$ sub-sequences $S^n_{(a)} :=
(\x_{(kl+a)\tau})_{k = 0}^{n-1}$ for $a = 0,\cdots,l-1$. The main idea of this
technique is that under fast $\upbeta$-mixing conditions, each sub-sequence can
be treated approximately as $n$ i.i.d. samples from the distribution
$\widetilde\Pi$ to which the classical generalization results may apply, with an
error that can be controlled by the mixing coefficient via the following lemma:
\begin{lemma}[{\cite[Proposition 2]{kuznetsov2017generalization}}] Let $h$ be
	any function on $\widetilde\Omega^n$ with $-M_1\leq h\leq M_2$ for
	$M_1,M_2\geq 0$. Then for any $0\leq a \leq l-1$, we have
	$$
		\left|\E_{\widetilde  S^n_{\widetilde\Pi}\sim{\widetilde \Pi}^{\otimes n}}\left[h(\widetilde S^n_{\widetilde\Pi})\right] - \E\left[h(\widetilde{S^n_{(a)}})\right]\right|\leq (M_0+M_1)n\upbeta(l\tau),
	$$
	where the second expectation is taken over the sub-$\sigma$-algebra of
	$\F_0^T$ generated by the sub-sequence $S^n_{(a)}=(\x_{(kl+a)\tau})_{k =
	0}^{n-1}$ and $\widetilde{S^n_{(a)}}:=(\widetilde\x_{(kl+a)\tau})_{k =
	0}^{n-1}$.
	\label{lem:subsample}
\end{lemma}
Based on Lemma \ref{lem:subsample}, we derive the following fast rate
generalization bound via local Rademacher complexity
arguments~\cite{bartlett2005local,koltchinskii2006local}. The proof is shown in
Appendix~\ref{sec:eg}.
\begin{theorem}
	Let $N = nl$. Suppose the localized Rademacher complexity satisfies $$
	\fR_N(\{\ell\circ g | g\in \fG, \E_\Pi[\ell\circ g] \leq r\}) \leq
	\phi(r),$$ where $\phi(r)$ is a sub-root function\footnote{A sub-root
	function $\phi(r)$ is a function $\phi:\R^+\to\R^+$ that is non-negative,
	non-decreasing function, satisfying that $\phi(r)/\sqrt{r}$ is
	non-increasing for $r>0$.} and $\fR_N(\fF)$ is the Rademacher complexity of
	a function class $\fF$ with respect to $N$ i.i.d. samples from the
	stationary distribution $\widetilde\Pi$, \emph{i.e.}
	\begin{equation}
		\fR_N(\fF) = \E_{ (\widetilde X_i)_{i=1}^N\sim{\widetilde{\Pi}}^{\otimes N},\boldsymbol{\sigma}\sim\unif(\{\pm 1\}^N)}\left[\sup_{f\in\fF}\dfrac{1}{N}\sum_{i=1}^N\sigma_i f(\widetilde X_i)\right].
	\end{equation}
 
     Let $r^*$ be the unique solution to the fixed-point equation $\phi(r)=r$.
	Then for any $\delta>N\upbeta(l\tau)$ and $\eps>0$, we have with probability
	$1-\delta$ for any $g\in\fG$,
	\begin{equation}
		\L_{\widetilde\Pi}^{g^0}(\hat g) \leq \dfrac{1}{1-\eps} \hat{\L}_N^{g^0}(\hat g) + \dfrac{176}{M^2\eps} r^* + \dfrac{\left(44\eps + 104 \right) M^2 \log\left(\frac{l}{\delta'}\right)}{\eps n},
	\end{equation}
	where $\delta' = \delta - N\upbeta(l\tau)$.
	\label{thm:hat}
\end{theorem}

Bias and noise in the objective function certainly affect the optimization.
Thus, we need to seek an oracle-type inequality for the expectation of the
population loss $\hat{\L}_N^{g^0}(\hat{g})$ over the data, which is proved in
Appendix \ref{sec:eo}. The main technique is a uniform martingale concentration
inequality (\emph{cf.} Lemma~\ref{lem:martingale}).  
\begin{theorem}
	Suppose $\fG$ is separable with respect to the $L^\infty$ norm with
	$\rho$-covering number $\N(\rho,\fG,\|\cdot\|_\infty)\geq 2$. Then under
	Assumption~\ref{ass:bias} and~\ref{ass:mart}, we have
	$$
		\begin{aligned}
            &\E\left[\hat{\L}_N^{g^0}(\hat{g})\right] \leq \dfrac{1+\eps}{1-\eps} \inf_{\bar{g}\in\fG}\E\left[ \hat{\L}_N^{g^0}(\bar g)\right]+\dfrac{3C_A}{\eps}\tau\\
            &\quad \quad +\dfrac{12 C_M \log\N(\rho,\fG,\|\cdot\|_\infty) }{\eps N\tau^\gamma}+2\sqrt{\dfrac{4C_M \rho^2\log 2}{N\tau^\gamma}},
        \end{aligned}
	$$
	where the expectation is taken over the sub-$\sigma$-algebra of $\F_0^T$
	generated by the trajectory $(\x_{k\tau})_{k=0}^N$ from which $\hat{g}$ is
	constructed by minimizing over the estimate empirical loss
	$\tilde{\L}_N^{g^0}(\hat{g};(\x_{k\tau})_{k=0}^N)$.
	\label{thm:oracle}
\end{theorem}

Especially, when we choose the hypothesis class $\fG$ as the sparse neural
network class $\fN(L,\p,S,M)$ and combine Theorem \ref{thm:hat} and Theorem
\ref{thm:oracle}, we have the following theorem with the proof given in
Appendix~\ref{sec:ea}:
\begin{theorem}
	Suppose Assumption~\ref{ass:bias} and~\ref{ass:mart} are satisfied and the
	ground truth $g^0\in\C^s(\Omega, M)$, and the hypothesis class $\fG =
	\fN(L,\p,S,M)$ with 
    $$
    \begin{aligned}
        K\asymp (N(\tau^\gamma\wedge 1))^{\frac{d}{2s+d}}&,\quad L \lesssim \log K,\\
        \|\p\|_\infty\lesssim K&,\quad S\lesssim K\log K.
    \end{aligned}
    $$ 
    Then with high probability the minimizer $\hat{g}$ obtained by minimizing
    the estimated empirical loss
    $\tilde{\L}_N^{g^0}(\hat{g};(\x_{k\tau})_{k=0}^N)$ satisfies
	\begin{equation}
        \E\left[\L_{\widetilde\Pi}^{g^0}(\hat{g})\right]\lesssim  \left(N(\tau^\gamma\wedge 1)\right)^{-\frac{2s}{2s+d}}\log^3 (N(\tau^\gamma\wedge 1))+\tau+ \dfrac{\log^2 N}{N\tau},
        \label{eq:generalrate}
    \end{equation}
    where the expectation has the same interpretation as in
    Theorem~\ref{thm:oracle}.
	\label{thm:thetheorem}
\end{theorem}

With Theorem~\ref{thm:thetheorem}, the detailed proofs of
Theorem~\ref{thm:drift} and~\ref{thm:diffusion} are given in
Appendix~\ref{sec:proofdrif} and~\ref{sec:proofdiff}, respectively.

\begin{figure*}[!htb]
    \centering
    \begin{subfigure}{0.4\textwidth}
        \includegraphics[width=\textwidth]{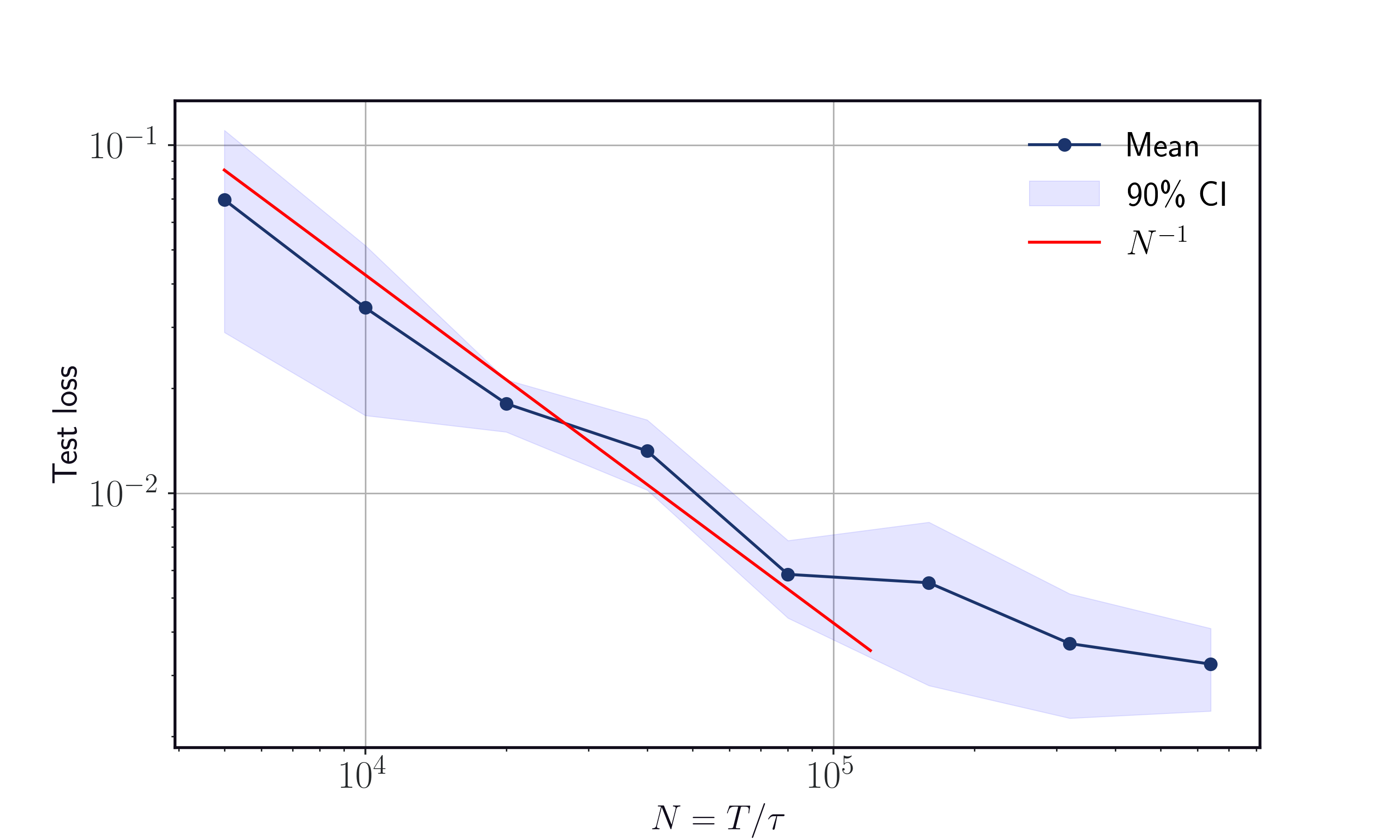}
        \caption{$\tau = 10^{-3}$ with varying $N$ and $T$}
        \label{fig:T}
    \end{subfigure}
    \begin{subfigure}{0.4\textwidth}
        \includegraphics[width=\textwidth]{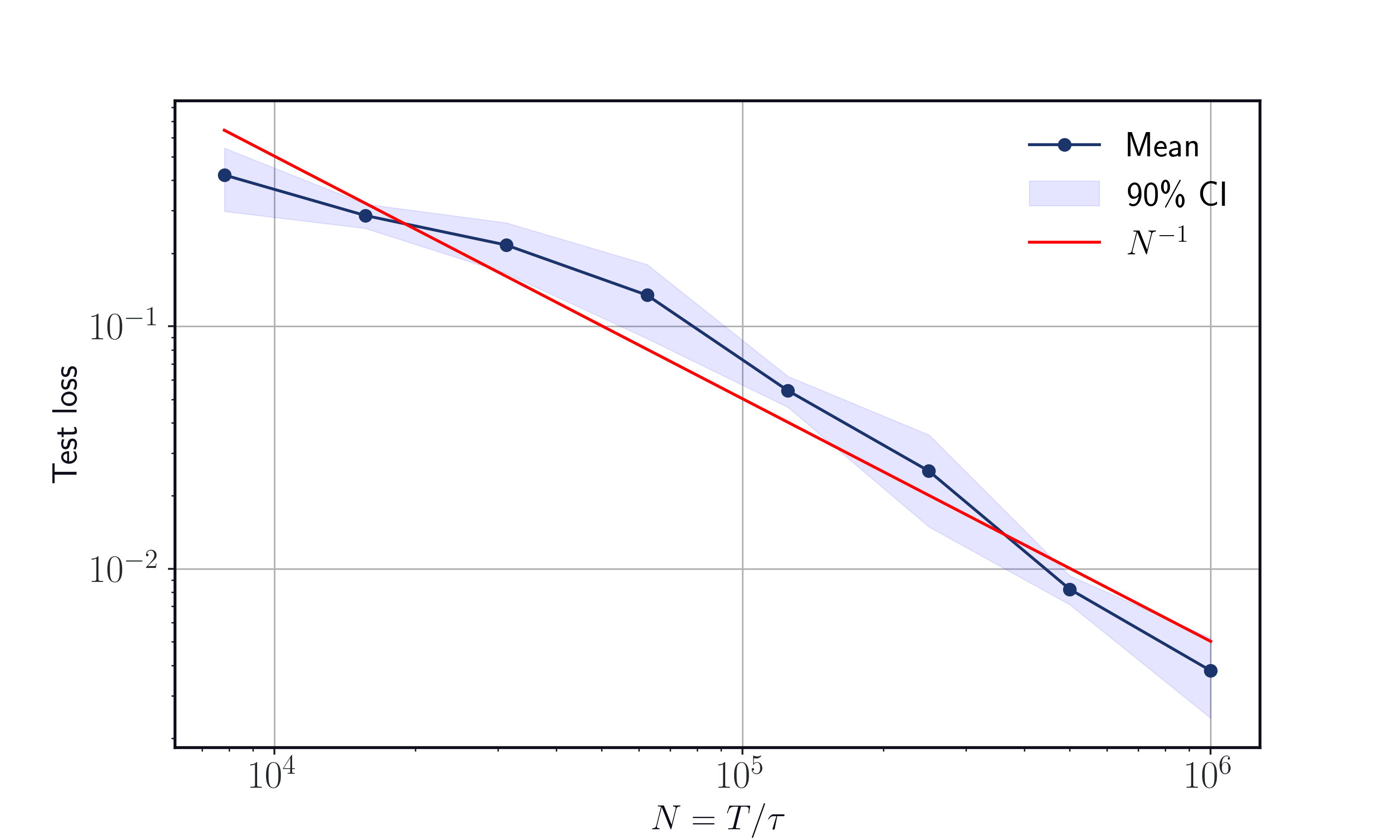}
        \caption{$T = 10^3$ with varying $N$ and $\tau$}
        \label{fig:tau}
    \end{subfigure}
    \caption{Numerical results of the proposed neural diffusion estimator are consistent with the scaling expected from the theoretical bound. We probe this by varying $N$ and $T=N\tau$ with fixed lag time $\tau$ and also by varying $N$ and $\tau=T/N$ with fixed time horizon $T$.} 
    \label{fig:conv}
\end{figure*}

\section{Experiments}

In this section, we present numerical results on a two-dimensional example, to
illustrate the accordance between our theoretical convergence rates and those of
our proposed neural diffusion estimator. Consider the following SDE in $\R^2$:
\begin{equation}
    \d \x_t = f(\x_t)\nabla f(\x_t)\d t + f(\x_t) \d \w_t,
    \label{eq:example}
\end{equation}
where $$f(\x) = 1+\frac{1}{2}\cos(2\pi(x_1+x_2)),$$ \emph{i.e.} $\b(\x) =
f(\x)\nabla f(\x)$ and $D(\x) = \frac{f(\x)^2}{2}I$, where $I$ is the $2\times
2$ identity matrix. Then it is straightforward to verify that this diffusion
process satisfies Assumption~\ref{ass:per},~\ref{ass:holder}, and~\ref{ass:reg}
with smoothness $s=\infty$. Our goal is to estimate the value of the function
$f(\x)$ within $\Omega = [0,1)^2$. We employ Algorithm~\ref{alg:alg} for
estimating both $\b(\x)$ and $D(\x)$ with separate neural networks and treat
them entirely independently in the inference task. One may also prove that the
stationary distribution $\widetilde{\Pi}$ of this diffusion process is given by
the Lebesgue measure on the two-dimensional torus, which makes evaluating errors
easier and more precise. 

To impose the periodic boundary, we introduce an explicit regularization term to
our training loss
$$
\L_{\rm per}(\hat{g}) = \E_{(\x,\y)\sim \unif(\partial\Omega^2), \widetilde{\x} = \widetilde{\y}}\left[ \left(\hat{g}(\x) - \hat{g}(\y)\right)^2 \right],
$$
approximated by $\hat{\L}_{\rm per}(\hat{g})$ with 1000 pairs of random samples
empirically. The final training loss is thus $\tilde{\L}^N(\hat{g}) + \lambda
\hat \L_{\rm per}(\hat{g})$, where $\lambda$ is a hyperparameter and $\hat{g}$
can be either $\hat\b$ or $\hat D$.

We first generate data using the Euler-Maruyama method with a time step $\tau_0
= 2\times10^{-5}$ up to $T_0 = 10^4$, and then sub-sample data at varying time
steps $\tau$ and time horizons $T$ for each experiment instance from this common
trajectory. We use a ResNet as our neural network structure with two residual
blocks, each containing a fully-connected layer with a hidden dimension of 1000.
Test data are generated by randomly selecting $5\times 10^4$ samples from
another sufficiently long trajectory, which are shared by all experiment
instances. The training process is executed on one Tesla V100 GPU.

According to our theoretical result (Theorem~\ref{thm:diffusion}), the
convergence rate of this implementation should be approximately of order
$N^{-1}+\tau + T^{-1}$ up to $\log$ terms. We thus consider two schemes in our
experiment. The first involves a fixed time step $\tau=10^{-3}$ with an expected
rate of $\tau + N^{-1}$, and the other maintains a fixed $T=10^3$ with an
expected rate of $N^{-1} + T^{-1}$. Each of the aforementioned instances is
carried out five times. Figure~\ref{fig:conv} presents the mean values along
with their corresponding confidence intervals from these runs. Additionally,
reference lines indicating the expected convergence rate $N^{-1}$ are shown in
red. Both schemes roughly exhibit the exponential decay phenomenon, aligning
with our theoretical expectations. As depicted in Figure~\ref{fig:T}, the
decreasing rate of the test error decelerates as $N$ exceeds a certain
threshold. This can be attributed to the fact that when $N$ and $T$ are
sufficiently large, the bias term $\tau$ arising from the discretization becomes
the dominant factor in the error. 

\section{Conclusion}

The ubiquity of correlated data in processes modeled with
spatially-inhomogeneous diffusions has created substantial barriers to analysis.
In this paper, we construct and analyze a neural network-based numerical
algorithm for estimating multidimensional spatially-inhomogeneous diffusion
processes based on discretely-observed data obtained from a single trajectory.
Utilizing $\upbeta$-mixing conditions and local Rademacher complexity arguments,
we establish the convergence rate for our neural diffusion estimator. Our upper
bound has recovered the minimax optimal nonparametric function estimation rate
in the common i.i.d. setting, even with correlated data. We expect our proof
techniques serve as a model for general exponential ergodic diffusion processes
beyond the toroidal setting considered here. Numerical experiments validate our
theoretical findings and demonstrate the potential of applying the neural
diffusion estimators across various contexts with provable accuracy guarantees.
Extending our results to typical biophysical settings, \emph{e.g.} compact
domains with reflective boundaries and motion blur due to measurement error,
could help establish more rigorous error estimates for physical inference
problems. 

\bibliography{references}

\appendix

\newpage

\onecolumn
\section{Detailed Proofs of Section~\ref{sec:proof}}
\label{sec:proof2}

Besides $\hat{\L}_N^{g^0}(\hat g)$ defined in the main text, we will use the
following notation to denote the empirical loss over a sample set $S$ another
than the sequence $S^N = (\x_{k\tau})_{k=0}^N$, $$\hat{\L}_{|S|}^{g^0}(\hat g;
S):= \dfrac{1}{|S|}\sum_{X_k\in S}\left( g^0(\widetilde X_k) -\hat{g}(\widetilde
X_k) \right)^2,$$ where applying $\widetilde{\cdot}$ is to make sure each
$\widetilde X_k\in\widetilde\Omega$. We will denote the squared loss $(g-g^0)^2$
by $\ell\circ g$, and adopt the shorthand notation $f_{k\tau} =
f(\widetilde{\x}(k\tau))$ for any function $f:\R^d\to \R$ throughout this
section for simplicity, when the sequence $S^N = (\x_{k\tau})_{k=0}^N$ is clear
from the context. 

\subsection{Proof of Theorem~\ref{thm:hat}}
\label{sec:eg}

In the following, we use $\widetilde  S^n_{\widetilde\Pi}$ to denote a sample
set consisting of $n$ i.i.d. samples drawn from the distribution $\widetilde
\Pi$, \emph{i.e.} $\widetilde  S^n_{\widetilde\Pi}=\{\widetilde X_1,\widetilde
X_2,\cdots,\widetilde X_n\}\sim{\widetilde \Pi}^{\otimes n}$.

To achieve the fast-rate generalization bound, we will make use of the following
local Rademacher complexity argument:
\begin{lemma}[{\cite[Theorem 3.3]{bartlett2005local}}] Let $\fF$ be a function
	class and for each $f\in\fF$, $M_1\leq f\leq M_2$ and $\var_{\widetilde \Pi}
	[f]\leq B \E_{\widetilde \Pi}[f]$ hold. Suppose $ \fR_N(\{f\in \fF |
	\E_{\widetilde \Pi}[f] \leq r\}) \leq \phi(r)$, where $\phi(r)$ is a
	sub-root function with a fixed point $r^*$. Then with probability at least
	$1-\delta$, we have for any $f\in \fF$
	$$
		\E_{\widetilde \Pi}[f]\leq \dfrac{1}{1-\eps}\dfrac{1}{n}\sum_{\widetilde X_k\in \widetilde S_{\widetilde \Pi}^n}f(\widetilde X_k)+\dfrac{704}{B\eps}r^*+\dfrac{\left(11(M_2-M_1)\eps+26B\right)\log\frac{1}{\delta}}{\eps n}.
	$$
	\label{lem:local}
\end{lemma}

\begin{remark}
    Compared with Rademacher complexity, the arguments of local Rademacher
    complexity enables adaptation to the local quadratic geometry of the loss
    function and recovers the fast rate generalization
    bound~\cite{bartlett2005local}. One of our contributions is showing how to
    adapt the local Rademacher arguments to the non-i.i.d data along the
    diffusion process and providing near-optimal bounds for diffusion
    estimation.
\end{remark}

For completeness, we also provide the proof of Lemma~\ref{lem:subsample}:
\begin{proof}[Proof of Lemma~\ref{lem:subsample}]
    Notice that each sub-sequence $S^n_{(a)} = (\x_{(kl+a)})_{k=0}^{n-1}$ is a
    sequence separated with time interval $l\tau$. Then the result follows
    directly from~\cite[Proposition 2]{kuznetsov2017generalization} by setting
    the distribution $\Pi$ and sub-sample $\boldsymbol Z^{(j)}$ as $\widetilde
    \Pi$ and $S^n_{(a)}$, respectively.
\end{proof}

\begin{remark}
    Lemma 5 serves as a crucial tool for addressing the challenges presented by
    non-i.i.d. data points. To give an intuitive explanation, consider the way
    we segment the initial observations, represented by $ S^N =
    (\boldsymbol{x}_{k\tau})_{k=0}^N$. We divide them into $l$ sub-sequences,
    denoted as $S^n_{(a)} = (\boldsymbol{x}_{(kl+a)\tau})_{k=0}^{n-1}$, for $a =
    0,\cdots, l-1$. In each of these sub-sequences, the observational gaps are
    $l\tau$, in contrast to the $\tau$ gap found in the initial sequence. By
    employing $\upbeta$-mixing and strategically choosing a sufficiently large
    value for $l$, we can effectively treat each sub-sequence as though it
    follows an i.i.d. distribution. The accuracy of this approximation is then
    determined by the mixing coefficient, as quantified in Lemma 5. With this
    setup, we can apply local Rademacher arguments to the approximated i.i.d.
    sequences, paving the way for fast convergence rates.
\end{remark}

With Lemma~\ref{lem:subsample} and Lemma~\ref{lem:local}, we are ready to
present the proof of Theorem~\ref{thm:hat}.
\begin{proof}[Proof of Theorem~\ref{thm:hat}]
	First, it is straightforward to check by Assumption~\ref{ass:holder} that
	for any $\widetilde \x\in \widetilde \Omega$,
	\begin{equation}
		|\ell\circ g^*(\widetilde \x) - \E_{\widetilde \Pi}[\ell\circ g^*(\widetilde \x)]|\leq 4M^2,
	\end{equation}
	and
	\begin{equation}
		\var_{\widetilde \Pi}\left[\ell\circ g^* \right]\leq \E_{\widetilde \Pi}\left[ \left(\ell\circ g^* \right)^2\right] = \E_{\widetilde \Pi}\left[ \left( g^*-g^0 \right)^4\right]\leq 4M^2 \E_{\widetilde \Pi}[\ell\circ g^*],
	\end{equation}
	Applying Lemma~\ref{lem:local} for $f = \ell\circ \hat{g}$ with $B = 4M^2$,
	$M_1 = 0$, and $M_2 = 4M^2$, yields that with probability at least
	$1-\delta'/l$, we have
	$$
		\L_{\widetilde \Pi}^{g^0}(\hat g)\leq  \dfrac{1}{1-\eps} \hat{\L}_n^{g^0}(\hat{g};\widetilde S^n_{\widetilde \Pi}) + \dfrac{176}{M^2\eps} r^* + \dfrac{\left(44 + \frac{104}{\eps} \right) M^2 \log\left(\frac{l}{\delta'}\right)}{n}.
	$$

	Now we split the empirical loss $ \hat{\L}_N^{g^0}(\hat{g})$ into the mean
	of $l$ empirical losses $ \hat{\L}_N^{g^0}(\hat{g}; S^n_{(a)})$, where
	$S^n_{(a)}$ are the sub-sequences of $S^N=(\x_{k\tau})_{k=0}^N$ obtained by
	sub-sampling:
	$$
		\begin{aligned}
			\L_{\widetilde \Pi}^{g^0}(\hat g)- \dfrac{1}{1-\eps} \hat{\L}_N^{g^0}(\hat{g}) & = \dfrac{1}{N}\sum_{k=0}^{N-1}\left(\L_{\widetilde \Pi}^{g^0}(\hat g)- \dfrac{1}{1-\eps}\ell\circ \hat{g}(\widetilde \x_{k\tau})\right) \\
        & = \dfrac{1}{l}\sum_{a=0}^{l-1}\left(\L_{\widetilde \Pi}^{g^0}(\hat g)- \dfrac{1}{1-\eps} \dfrac{1}{n}\sum_{k=0}^{n-1} \ell\circ \hat{g}(\widetilde \x_{(kl+a)\tau})\right) \\
			& = \dfrac{1}{l}\sum_{a = 0}^{l-1} \left(\L_{\widetilde \Pi}^{g^0}(\hat g)- \dfrac{1}{1-\eps}\hat{\L}_n^{g^0}(\hat{g};S^n_{(a)})\right),
		\end{aligned}
	$$
    and for each summand, we apply Lemma~\ref{lem:subsample} to the indicator
    function of the event $\left\{\L_{\widetilde \Pi}^{g^0}(\hat g)-
    \frac{1}{1-\eps} \hat{\L}_n^{g^0}(\hat{g};S^n_{(a)})>\alpha\right\}$ and
    obtain
    $$
    \P\left(\L_{\widetilde \Pi}^{g^0}(\hat g)- \dfrac{1}{1-\eps} \hat{\L}_n^{g^0}(\hat{g};S^n_{(a)})>\alpha\right)\leq \P\left(\L_{\widetilde \Pi}^{g^0}(\hat g)- \dfrac{1}{1-\eps} \hat{\L}_n^{g^0}(\hat{g};\widetilde S^n_{\widetilde \Pi})>\alpha\right) +n\upbeta(l\tau).
    $$

    Then let $\alpha =\frac{176}{M^2\eps} r^* + \frac{\left(44 +
    \frac{104}{\eps} \right) M^2 \log\left(\frac{l}{\delta'}\right)}{n}$, we
    have by the union bound
	$$
		\begin{aligned}
			\P\left(\L_{\widetilde \Pi}^{g^0}(\hat g)- \dfrac{1}{1-\eps} \hat{\L}_N^{g^0}(\hat{g})> \alpha\right)
			 & \leq \P\left( \dfrac{1}{l}\sum_{a = 0}^{l-1} \left(\L_{\widetilde \Pi}^{g^0}(\hat g)- \dfrac{1}{1-\eps} \hat{\L}_n^{g^0}(\hat{g};S^n_{(a)})\right)> \alpha\right)    \\
			 & \leq \sum_{a = 0}^{l-1}\left(\P\left(\L_{\widetilde \Pi}^{g^0}(\hat g)- \dfrac{1}{1-\eps} \hat{\L}_n^{g^0}(\hat{g};S^n_{(a)})>\alpha\right) \right)                  \\
			 & \leq \sum_{a = 0}^{l-1}\left(\P\left(\L_{\widetilde \Pi}^{g^0}(\hat g)- \dfrac{1}{1-\eps} \hat{\L}_n^{g^0}(\hat{g};\widetilde S^n_{\widetilde \Pi})>\alpha\right) +n\upbeta(l\tau) \right) \\
			 & \leq l\P\left(\L_{\widetilde \Pi}^{g^0}(\hat g)- \dfrac{1}{1-\eps} \hat{\L}_N^{g^0}(\hat{g};\widetilde S^n_{\widetilde \Pi})> \alpha\right) + N\upbeta(l\tau)                              \\
			 & \leq \delta' + N\upbeta(l\tau) = \delta,
		\end{aligned}
	$$
	and the result follows.
\end{proof}

\subsection{Proof of Theorem~\ref{thm:oracle}}
\label{sec:eo}

In this section, the expectiation $\E$ should be referred to as being taken over
the sub-$\sigma$-algebra of $\F_0^T$ generated by the trajectory
$(\x_{k\tau})_{k=0}^N$ from which $\hat{g}$ is constructed by minimizing over
the estimate empirical loss $\tilde{\L}_N^{g^0}(\hat{g};(\x_{k\tau})_{k=0}^N)$.

We first consider the following decomposition of the expectation of the
empirical loss $\hat{\L}_N^{g^0}(\hat{g})$:
\begin{lemma} For an arbitrary $\bar{g}\in \fG$, we have
	\begin{equation}
		\E\left[\hat{\L}_N^{g^0}(\hat{g})\right] \leq \E\left[\hat{\L}_N^{g^0}(\bar{g})\right] + \dfrac{2}{N}\E\left[\sum_{k=0}^{N-1} \Delta A_{k\tau} \left(\hat{g}_{k\tau}-\bar{g}_{k\tau}\right)\right] + \dfrac{2}{N}\E\left[\sum_{k=0}^{N-1} \Delta M_{k\tau} \left(\hat{g}_{k\tau}-g^0_{k\tau}\right)\right].
		\label{eq:decomp}
	\end{equation}
	\label{lem:decomp}
\end{lemma}
\begin{proof}
	By the definition of the estimator $\hat{g}$, we have
	$\tilde{\L}_N^{g^0}(\hat{g};(\x_{k\tau})_{k=0}^N)\leq
	\tilde{\L}_N^{g^0}(\bar{g};(\x_{k\tau})_{k=0}^N)$ for every possible
	sequence $(\x_{k\tau})_{k=0}^N$. Recall that
	$$
    \tilde{\L}_N^{g^0}(g;(\x_{k\tau})_{k=0}^N) = \dfrac{1}{N}\sum_{k=0}^{N-1} \left( g^0_{k\tau}+\Delta Z_{k\tau} -g_{k\tau} \right)^2 = \hat{\L}_N^{g^0}(g)+\dfrac{2}{N} \sum_{k=0}^{N-1}\Delta Z_{k\tau} (g^0_{k\tau} - g_{k\tau})+\dfrac{1}{N} \sum_{k=0}^{N-1}(\Delta Z_{k\tau})^2,
	$$
	we have by taking expectation to both sides,
	$$
		\begin{aligned}
			& \E\left[\hat{\L}_N^{g^0}(\hat{g})\right]-\E\left[\hat{\L}_N^{g^0}(\bar{g})\right]\\
			=    & \E\left[\tilde{\L}_N^{g^0}(\hat{g};(\x_{k\tau})_{k=0}^N)\right] + \E\left[\dfrac{2}{N} \sum_{k=0}^{N-1}\Delta Z_{k\tau} \left( \hat{g}_{k\tau} - g^0_{k\tau} \right)\right]-  \E\left[\dfrac{1}{N} \sum_{k=0}^{N-1} (\Delta Z_{k\tau})^2 \right]  \\
			-    & \E\left[\tilde{\L}_N^{g^0}(\bar{g};(\x_{k\tau})_{k=0}^N)\right] - \E\left[\dfrac{2}{N} \sum_{k=0}^{N-1} \Delta Z_{k\tau} \left( \bar{g}_{k\tau} - g^0_{k\tau} \right)\right] + \E\left[\dfrac{1}{N} \sum_{k=0}^{N-1} (\Delta Z_{k\tau})^2 \right] \\
			\leq & \dfrac{2}{N}\E\left[\sum_{k=0}^{N-1} \Delta Z_{k\tau} \left( \hat{g}_{k\tau} - \bar{g}_{k\tau} \right)\right],
		\end{aligned}
	$$
    where we used the fact that the noise $(\Delta Z_{k\tau})_{k=0}^{N-1}$ only
    depends on the ground truth $g^0$ and the sequence $(\x_{k\tau})_{k=0}^N$.

	Since $(M_{k\tau})_{k=0}^N$ forms a martingale, we have $$
	\E\left[\sum_{k=0}^{N-1} \Delta  M_{k\tau} \left( \bar{g}_{k\tau} -
	g^0_{k\tau} \right)\right] = \E\left[\sum_{k=0}^{N-1} \E\left[\Delta
	M_{k\tau} \left( \bar{g}_{k\tau} - g^0_{k\tau} \right)|
	\F_0^{k\tau}\right]\right] = 0,$$ and thus
	$$
		\begin{aligned}
			  & \E\left[\sum_{k=0}^{N-1}  \Delta Z_{k\tau}  \left( \hat{g}_{k\tau} - \bar{g}_{k\tau} \right)\right] \\
			= & \E\left[\sum_{k=0}^{N-1} \Delta A_{k\tau} \left( \hat{g}_{k\tau} - \bar{g}_{k\tau} \right)\right] + \E\left[\sum_{k=0}^{N-1} \Delta M_{k\tau} \left( \hat{g}_{k\tau} - \bar{g}_{k\tau} \right)\right] +  \E\left[\sum_{k=0}^{N-1} \Delta  M_{k\tau} \left( \bar{g}_{k\tau} - g^0_{k\tau} \right)\right] \\
			= & \E\left[\sum_{k=0}^{N-1} \Delta A_{k\tau} \left( \hat{g}_{k\tau} - \bar{g}_{k\tau} \right)\right] + \E\left[\sum_{k=0}^{N-1} \Delta M_{k\tau} \left( \hat{g}_{k\tau} - g^0_{k\tau} \right)\right],
		\end{aligned}
	$$
	and the result follows.
\end{proof}

Following Remark~\ref{rem:noise} in the main text, we will refer to the second
term in the RHS of~\eqref{eq:decomp} as the bias term and the last as the
martingale noise term.

For the bias term, we have the following simple bound:
\begin{lemma}
	For any $\eps>0$, we have
	$$\E\left[\dfrac{1}{N}\sum_{k=0}^{N-1} \Delta A_{k\tau}
	\left(\hat{g}_{k\tau}-\bar{g}_{k\tau}\right)\right]\leq
	\dfrac{\eps}{4}\E\left[\hat{\L}_N^{g^0}(\hat{g})\right]+\dfrac{\eps}{2}\E\left[\hat{\L}_N^{g^0}(\bar{g})\right]+\dfrac{3}{2\eps}\E\left[\dfrac{1}{N}\sum_{k=0}^{N-1}
	(\Delta A_{k\tau})^2 \right].$$
	\label{lem:bias}
\end{lemma}
\begin{proof}
	$$
		\begin{aligned}
			     & \E\left[\dfrac{1}{N}\sum_{k=0}^{N-1} \Delta A_{k\tau} \left(\hat{g}_{k\tau}-\bar{g}_{k\tau}\right)\right]                                                                                                                                               \\
			=    & \E\left[\dfrac{1}{N}\sum_{k=0}^{N-1} \left(\Delta A_{k\tau} \left(\hat{g}_{k\tau}-g^0_{k\tau}\right)+ \Delta A_{k\tau}\left(g^0_{k\tau}-\bar{g}_{k\tau}\right)\right)\right]                                                                            \\
			\leq & \E\left[\dfrac{1}{N}\sum_{k=0}^{N-1} \left(\dfrac{1}{\eps}(\Delta A_{k\tau})^2 +\dfrac{\eps}{4}\left(\hat{g}_{k\tau}-g^0_{k\tau}\right)^2+\dfrac{1}{2\eps}(\Delta A_{k\tau})^2 +\dfrac{\eps}{2}\left(g^0_{k\tau}-\bar{g}_{k\tau}\right)^2\right)\right] \\
			=    & \dfrac{\eps}{4}\E\left[\hat{\L}_N^{g^0}(\hat{g})\right]+\dfrac{\eps}{2}\E\left[\hat{\L}_N^{g^0}(\bar{g})\right]+\dfrac{3}{2\eps}\E\left[\dfrac{1}{N}\sum_{k=0}^{N-1} (\Delta A_{k\tau})^2 \right],
		\end{aligned}
	$$
	where the inequality is by AM-GM and the last equality is due to
	Assumption~\ref{ass:bias}.
\end{proof}

The following proof of the martingale noise term bound is inspired by the proof
of~\cite[Lemma 4]{schmidt2020nonparametric}, where i.i.d. Gaussian noise is
considered for nonparametric regression.
\begin{definition}[$\eps$-covering]
	We define the $\eps$-covering $\fF_\eps$ of a function class $\fF$ with
	respect to the norm $\|\cdot\|$ as the subset of $\fF$ with the smallest
	possible cardinality such that for any $f\in\fF$, there exists an $f'\in
	\fF_\eps$ satisfying $\|f'-f\|\leq \eps$. The cardinality of $\fF_\eps$ is
	denoted by $\N(\epsilon,\fF,\|\cdot\|)$.
\end{definition}
For a fixed $\rho>0$, let $\fG_\rho$ be the $\rho$-net of the function class
$\fG$ with respect to the norm $\|\cdot\|_\infty$, where $\rho>0$. Without loss
of generality, we assume that $\N(\rho,\fG,\|\cdot\|_\infty)\geq 2$.

We now introduce the following exponential inequality for local martingales.
\begin{lemma}
	Let $L_t$ be a continuous local martingale and $\qv{L}_t$ be the
	corresponding quadratic variation of $L_t$. If
	$\mathbb{E}\left[\sqrt{\langle L\rangle_t}\right]>0$, then for any $t, y>0$,
	$$
		\mathbb{E}\left[\dfrac{y}{\sqrt{\langle L\rangle_t+y^2}} \exp \left(\frac{L_t^2}{2\left(\langle L\rangle_t+y^2\right)}\right)\right] \leq 1.
	$$
	Moreover,  then
	$$
		\mathbb{E}\left[\exp \left(\dfrac{L_t^2}{4\left(\langle L\rangle_t+\mathbb{E}\left[\langle L\rangle_t\right]\right)}\right)\right] \leq \sqrt{2} .
	$$
	\label{lem:martingale}
\end{lemma}
\begin{proof}
    The result follows directly from~\cite[Theorem 2.1]{de2004self} by setting
    $A = L_t$ and $B = \sqrt{\qv{L}_t}$, and noticing that $$ \E\left[\exp\left(
    \lambda L_t -
    \dfrac{\lambda^2}{2}\qv{L}_t\right)\right]\leq\E\left[\exp\left( \lambda L_0
    - \dfrac{\lambda^2}{2}\qv{L}_0\right)\right]= 1$$ holds for any $t\geq 0$
    and $\lambda\in\R$ by~\cite[Lemma 1.2]{de2004self}. Especially, the second
    inequality follows from
    $$
    \E\left[\exp \left(\dfrac{L_t^2}{4\left(\langle L\rangle_t+\E\left[\langle L\rangle_t\right]\right)}\right)\right] \leq \E\left[\exp \left(\dfrac{L_t^2}{4\left(\langle L\rangle_t+\left(\E\left[\sqrt{\qv{L}_t}\right]\right)^2\right)}\right)\right] \leq  \sqrt{2} .
    $$
\end{proof}

Lemma~\ref{lem:martingale} enables us to give the following concentration
inequality for the martingale noise term.
\begin{lemma}
	Let $g'\in\fG_\rho$ such that $\|\hat{g}-g'\|_\infty\leq \rho$, then for any
	$\eps>0$ and $g\in\fG_\rho$, we have
	\begin{equation}
		\E\left[\dfrac{1}{N}\sum_{k=0}^{N-1} \Delta M_{k\tau} \left( g'_{k\tau} - g^0_{k\tau} \right)\right]\leq \dfrac{6 \max_k\left|\E\left[\Delta\qv{M}_{ k\tau} \big| \F_0^{k\tau}\right]\right| \log\N(\rho,\fG,\|\cdot\|_\infty) }{\eps N}  + \dfrac{\eps}{4}\E\left[\hat{L}_N^{g^0}(g')\right].
		\label{eq:mart1}
	\end{equation}
	\label{lem:mart1}
\end{lemma}
\begin{proof}
	First, for any $g\in \fG_\rho$, we define the following local martingale
	$$L[g-g^0]_t=\sum_{k=0}^{N-1}(g_{k\tau} - g^0_{k\tau}) \left( M_{t\wedge
	(k+1)\tau} -  M_{t\wedge k\tau} \right),$$ for any $t\geq 0$ so that the LHS
	of~\eqref{eq:mart1} is $\frac{1}{N}\E\left[L[g-g^0]_T\right]$.

	By martingale representation theorem, $(M_t)_{t\geq 0}$ as a continuous
	local martingale is an It\^{o} integral and thus the corresponding quadratic
	variation $(\qv{M}_t)_{t\geq 0}$ exists. As a result, the quadratic
	variation of $L[g-g^0]$ is
	$$
		\langle L[g-g^0]\rangle_t = \sum_{k=0}^{N-1} (g_{k\tau} - g^0_{k\tau})^2 \left( \qv{M}_{t\wedge (k+1)\tau} - \qv{M}_{t\wedge k\tau} \right).
	$$
	Take $y = \sqrt{\E\left[\qv{L[g-g^0]}_T\right]+\eta}$, where $\eta>0$, we
	have by Lemma~\ref{lem:martingale},
	$$
		\E\left[\dfrac{y}{\sqrt{\qv{L[g-g^0]}_T+y^2}}\exp\left(\dfrac{L[g-g^0]_T^2}{2\left(\qv{L[g-g^0]}_T+y^2\right)}\right)\right] \leq 1.
	$$
	Therefore, by Jensen's inequality and Cauchy-Schwarz, we have
	$$
		\begin{aligned}
			     & \exp\E\left[\dfrac{L[g'-g^0]_T^2}{4\left(\qv{L[g'-g^0]}_T+y^2\right)}\right]
			\leq \E\left[
				\exp\left(\dfrac{L[g'-g^0]_T^2}{4\left(\qv{L[g'-g^0]}_T+y^2\right)}\right)
			\right]\\
			\leq & \sqrt{\E\left[\dfrac{y}{\sqrt{\qv{L[g'-g^0]}_T+y^2}}\exp\left(\dfrac{L[g'-g^0]_T^2}{2\left(\qv{L[g'-g^0]}_T+y^2\right)}\right)\right] \E\left[\dfrac{\sqrt{\qv{L[g'-g^0]}_T+y^2}}{y}\right] } \\
			\leq & \sqrt{\sum_{g\in\fG_\rho}
				\E\left[\dfrac{y}{\sqrt{\qv{L[g-g^0]}_T+y^2}}\exp\left(\dfrac{L[g-g^0]_T^2}{2\left(\qv{L[g-g^0]}_T+y^2\right)}\right)\right]
			\E\left[\dfrac{\sqrt{2y^2-\eta}}{y}\right] }\\
			\leq & \sqrt{\N(\rho,\fG,\|\cdot\|_\infty) \sqrt{2} },
		\end{aligned}
	$$
    where the second to last inequality is due to the fact $g'\in \fG_\rho$ and
    thus the first expectation must be bounded by the summation over all
    possible $g\in\fG_\rho$.

	Again by Cauchy-Schwarz,
	$$
		\begin{aligned}
			\dfrac{1}{N}\E\left[L[g'-g^0]_T\right] & =\dfrac{1}{N}\E\left[\dfrac{L[g'-g^0]_T}{2\sqrt{\qv{L[g'-g^0]}_T+y^2}}2\sqrt{\qv{L[g'-g^0]}_T+y^2} \right]                                                                        \\
			     & \leq \sqrt{\dfrac{4}{N^2}\E\left[\dfrac{L[g'-g^0]_T^2}{4\left(\qv{L[g'-g^0]}_T+y^2\right)}\right]\E\left[\qv{L[g'-g^0]}_T+y^2\right] }                                            \\
			     & \leq \sqrt{\dfrac{2\log\N(\rho,\fG,\|\cdot\|_\infty)+\log 2 }{N^2} \E\left[2\sum_{k=0}^{N-1} \left(g'_{k\tau}-g^0_{k\tau}\right)^2  \Delta \qv{M}_{ k\tau} +\eta \right] }.
		\end{aligned}
	$$

	Finally, let $\eta\to 0 $,
	\begin{equation*}
		\begin{aligned}
			\dfrac{1}{N}\E\left[L[g'-g^0]_T\right] & \leq \sqrt{\dfrac{3\log\N(\rho,\fG,\|\cdot\|_\infty)}{N^2} 2\E\left[\sum_{k=0}^{N-1} \left(g'_{k\tau}-g^0_{k\tau}\right)^2  \Delta \qv{M}_{ k\tau}  \right] }          \\
			                                      & \leq \sqrt{\dfrac{6\log\N(\rho,\fG,\|\cdot\|_\infty) }{N}    \max_k\left|\E\left[\Delta\qv{M}_{ k\tau} \big| \F_0^{k\tau}\right]\right|
			\E\left[\dfrac{1}{N}\sum_{k=0}^{N-1} \left(g'_{k\tau}-g^0_{k\tau}\right)^2   \right] }                                                                                                                         \\
			                                      & \leq \dfrac{6 \max_k\left|\E\left[\Delta\qv{M}_{ k\tau} \big| \F_0^{k\tau}\right]\right| \log\N(\rho,\fG,\|\cdot\|_\infty) }{\eps N}  + \dfrac{\eps}{4} \E\left[\hat{L}_N^{g^0}(g')\right],
		\end{aligned}
		\label{eq:Lg}
	\end{equation*}
	where the second inequality is due to the tower property
	\begin{equation}
		\begin{aligned}
			\E\left[\sum_{k=0}^{N-1} \left(g'_{k\tau}-g^0_{k\tau}\right)^2  \Delta \qv{M}_{ k\tau}  \right] & = \sum_{k=0}^{N-1} \E\left[ \left(g'_{k\tau}-g^0_{k\tau}\right)^2  \Delta \qv{M}_{ k\tau}  \right]                                                        \\
			                                                                                               & = \E\left[\sum_{k=0}^{N-1} \E\left[ \left(g'_{k\tau}-g^0_{k\tau}\right)^2  \Delta \qv{M}_{ k\tau} \big| \F_0^{k\tau}  \right]\right]                      \\
			                                                                                               & = \E\left[\sum_{k=0}^{N-1} \left(g'_{k\tau}-g^0_{k\tau}\right)^2 \E\left[\Delta\qv{M}_{ k\tau} \big| \F_0^{k\tau}\right]  \right]                         \\
			                                                                                               & \leq  \max_k\left|\E\left[\Delta\qv{M}_{ k\tau} \big| \F_0^{k\tau}\right]\right| \E\left[\sum_{k=0}^{N-1}  \left(g'_{k\tau}-g^0_{k\tau}\right)^2  \right]
		\end{aligned}
		\label{eq:tower}
	\end{equation}
	and the last inequality is by AM-GM.
\end{proof}

It remains to bound the error by projecting the estimator $\hat{g}$ to the
$\rho$-covering $\fG_\rho$, which is given by the following lemma.
\begin{lemma}
	Let $g'\in\fG_\rho$ such that $\|\hat{g}-g'\|_\infty\leq \rho$, then
	$$
		\E\left[\dfrac{1}{N}\sum_{k=0}^{N-1} \Delta M_{k\tau} \left( \hat{g}_{k\tau} - g'_{k\tau} \right)\right]\leq \sqrt{\dfrac{4\max_k\left|\E\left[\Delta\qv{M}_{ k\tau} \big| \F_0^{k\tau}\right]\right| \rho^2\log 2}{N}}.
	$$
	\label{lem:mart2}
\end{lemma}
\begin{proof}
	We will follow the notations in the proof of Lemma~\ref{lem:mart1}. If
    $\qv{L[\hat{g}-g']}_T = 0$, $L[\hat{g}-g']_t = 0$ almost everywhere and the
    result holds trivially. 
    
    Now suppose $\qv{L[\hat{g}-g']}_T>0$, applying Lemma~\ref{lem:martingale}
    and Jensen's inequality yields
	$$
		\begin{aligned}
			     & \E\left[\left( \dfrac{L[\hat{g}-g']_T^2}{4\left(\qv{L[\hat{g}-g']}_T+\E\left[\qv{L[\hat{g}-g']}_T\right]\right)} \right) \right]                                  \\
			\leq & \log  \E\left[\exp\left( \dfrac{L[\hat{g}-g']_T^2}{4\left(\qv{L[\hat{g}-g']}_T+\E\left[\qv{L[\hat{g}-g']}_T\right]\right)} \right) \right]\leq \dfrac{\log 2}{2}.
		\end{aligned}
	$$
	With a similar argument as in~\eqref{eq:tower}, we have
	$$
		\begin{aligned}
			&\E\left[\qv{L[\hat{g}-g']}_T\right]=  \E\left[\sum_{k=0}^{N-1} \left( \hat{g}_{k\tau} - g'_{k\tau} \right)^2  \Delta \qv{M}_{ k\tau} \right]                                                     \\
			\leq                                 & \max_k\left|\E\left[\Delta\qv{M}_{ k\tau} \big| \F_0^{k\tau}\right]\right| \E\left[\sum_{k=0}^{N-1} \left( \hat{g}_{k\tau} - g'_{k\tau} \right)^2  \right] 
			\leq                                 N\max_k\left|\E\left[\Delta\qv{M}_{ k\tau} \big| \F_0^{k\tau}\right]\right| \rho^2.
		\end{aligned}
	$$
	Therefore, we have
	$$
		\begin{aligned}
			\dfrac{1}{N}\E\left[L[\hat{g}-g']_T\right] & =\dfrac{1}{N}\E\left[\dfrac{L[\hat{g}-g']_T}{2\sqrt{\qv{L[\hat{g}-g']}_T+\E\left[\qv{L[\hat{g}-g']}_T\right]}}2\sqrt{\qv{L[\hat{g}-g']}_T+\E\left[\qv{L[\hat{g}-g']}_T\right]} \right] \\
			                                           & \leq \dfrac{2}{N}\sqrt{\E\left[\dfrac{L[\hat{g}-g']_T^2}{4\left(\qv{L[\hat{g}-g']}_T+\E\left[\qv{L[\hat{g}-g']}_T\right]\right)}\right]2\E\left[\qv{L[\hat{g}-g']}_T\right] }          \\
			                                           & \leq \sqrt{\dfrac{4\max_k\left|\E\left[\Delta\qv{M}_{ k\tau} \big| \F_0^{k\tau}\right]\right| \rho^2\log 2}{N}}.
		\end{aligned}
	$$

\end{proof}

Now we are ready to present the proof of Theorem~\ref{thm:oracle}.
\begin{proof}[Proof of Theorem~\ref{thm:oracle}]
	Plug Lemma~\ref{lem:bias},~\ref{lem:mart1}, and~\ref{lem:mart2} into
	Lemma~\ref{lem:decomp}, we have
	$$
		\begin{aligned}
			\E\left[\hat{\L}_N^{g^0}(\hat{g})\right] \leq & \E\left[\hat{\L}_N^{g^0}(\bar{g})\right] + \dfrac{2}{N}\E\left[\sum_{k=0}^{N-1} \Delta A_{k\tau} \left(\hat{g}_{k\tau}-\bar{g}_{k\tau}\right)\right] + \dfrac{2}{N}\E\left[\sum_{k=0}^{N-1} \Delta M_{k\tau} \left(\hat{g}_{k\tau}-g^0_{k\tau}\right)\right] \\
			\leq                                    & \E\left[\hat{\L}_N^{g^0}(\bar{g})\right] + \dfrac{2}{N}\E\left[\sum_{k=0}^{N-1} \Delta A_{k\tau} \left(\hat{g}_{k\tau}-\bar{g}_{k\tau}\right)\right]                                                                                                         \\
			                                        & + \dfrac{2}{N}\E\left[\sum_{k=0}^{N-1} \Delta M_{k\tau} \left(\hat{g}_{k\tau}-g'_{k\tau}\right)\right]
			+\dfrac{2}{N}\E\left[\sum_{k=0}^{N-1} \Delta M_{k\tau} \left(g'_{k\tau}-g^0_{k\tau}\right)\right]                                                                                                                                                                                                \\
			\leq                                    & \E\left[\hat{\L}_N^{g^0}(\bar{g})\right]+ \dfrac{\eps}{2}\E\left[\hat{\L}_N^{g^0}(\hat{g})\right]+\eps\E\left[\hat{\L}_N^{g^0}(\bar{g})\right]+\dfrac{3}{\eps}\E\left[\dfrac{1}{N}\sum_{k=0}^{N-1} (\Delta A_{k\tau})^2 \right]
			\\
			                                        & +\dfrac{12 \max_k\left|\E\left[\Delta\qv{M}_{ k\tau} \big| \F_0^{k\tau}\right]\right| \log\N(\rho,\fG,\|\cdot\|_\infty) }{\eps N}  + \dfrac{\eps}{2} \hat{L}^N(g)                                                                                      \\
			                                        & +2\sqrt{\dfrac{4\max_k\left|\E\left[\Delta\qv{M}_{ k\tau} \big| \F_0^{k\tau}\right]\right| \rho^2\log 2}{N}}.
		\end{aligned}
	$$
	Then we rearrange the terms and take infimum over all $\bar{g},\in\fG$,
	$$
		\begin{aligned}
			 & \E\left[\hat{\L}_N^{g^0}(\hat{g})\right] \leq\dfrac{1+\eps}{1-\eps} \inf_{\bar{g}\in\fG}\E\left[\hat{\L}_N^{g^0}(\bar{g})\right]+\dfrac{3}{\eps}\E\left[\dfrac{1}{N}\sum_{k=0}^{N-1} (\Delta A_{k\tau})^2 \right]                                         \\
			 & +\dfrac{12 \max_k\left|\E\left[\Delta\qv{M}_{ k\tau} \big| \F_0^{k\tau}\right]\right| \log\N(\rho,\fG,\|\cdot\|_\infty) }{\eps N}+2\sqrt{\dfrac{4\max_k\left|\E\left[\Delta\qv{M}_{ k\tau} \big| \F_0^{k\tau}\right]\right| \rho^2\log 2}{N}} \\
			 & \leq \dfrac{1+\eps}{1-\eps} \inf_{\bar{g}\in\fG}\E\left[\hat{\L}_N^{g^0}(\bar{g})\right]+\dfrac{3C_A}{\eps}\tau+\dfrac{12 C_M \log\N(\rho,\fG,\|\cdot\|_\infty) }{\eps N\tau^\gamma}+2\sqrt{\dfrac{4C_M \rho^2\log 2}{N\tau^\gamma}} ,
		\end{aligned}
	$$
	where the last inequality is by Assumption~\ref{ass:bias}
	and~\ref{ass:mart}.

\end{proof}

\subsection{Proof of Theorem~\ref{thm:thetheorem}}
\label{sec:ea}

Before we present the proof, we introduce the following lemma concerning the
complexity of the sparse neural network function class, which will serve as the
fundamental building block of our proof via local Rademacher arguments.
\begin{lemma}[$\log$-covering number of $\fN(L,\p,S,M)$~{\cite[Lemma
5]{schmidt2020nonparametric}}]
	$$
		\log \N(\rho,\fN(L,\p,S,\infty),\|\cdot\|_\infty) \leq (S+1) \log \left(2^{2L+5}\rho^{-1}(L+1)d^2 (S+1)^{2L}\right).
	$$
	\label{lem:covering}
\end{lemma}

In fact, the Rademacher complexity $\fR_N(\fF)$ can be bounded by
$\log$-covering number via the following lemma.
\begin{lemma}[Localized Dudley's theorem]
	For any function class $\fF$,
	$$
		\fR_N(\fF) \leq \E_{\mu^{\otimes N}}\left[\inf_{\rho>0} \left(4\rho + 12 \int_\rho^\infty\sqrt{\dfrac{\log \N(u,\fF,\|\cdot\|_{2,S})}{N}}\d u\right)\right],
	$$
	where the expectation in the definition of $\fR_N(\fF)$ is taken over a
	sample set $S$ of $N$ i.i.d. samples $X_1,\cdots,X_N$ drawn from a
	distribution $\mu$, and $\|\cdot\|_{2,S}$ denotes the $L^2$ norm with
	respect to the empirical measure $\frac{1}{N}\sum_{i=1}^N \delta(x - X_i)$.
\end{lemma}

The following lemma will also be used in the proof of
Theorem~\ref{thm:thetheorem}.
\begin{lemma}[Talagrand's contraction lemma]
	Let $\phi:\R\to\R$ be a $L$-Lipschitz continuous function and $\fF$ be a
	function class, then
	$$
		\fR_N(\fF) \leq L \fR_N(\phi\circ\fF),
	$$
	where $\phi\circ\fF = \{\phi\circ f | f\in\fF\}$.
\end{lemma}

With all lemmas aforementioned, we first bound the local Rademacher complexity
of the sparse neural network class $\fN(L,\p,S,M)$.
\begin{lemma}[Local Rademacher complexity of $\fN(L,\p,S,M)$]
	The local Rademacher complexity of the sparse neural network class
	$\fN(L,\p,S,M)$
	$$
		\fR_N(\{\ell\circ g | g\in \fN(L,\p,S,M), \E_{\widetilde \Pi}[\ell\circ g] \leq r\})
	$$
	as appeared in Theorem~\ref{thm:hat} is bounded by the sub-root function
	$$
		\begin{aligned}
			\phi(r)\leq & \dfrac{32M}{N}+ 96 M\sqrt{\dfrac{ r(S+1) \log \left(2^{2L+5}N(L+1)d^2 (S+1)^{2L}\right) }{N}} \\
			            & + \dfrac{ 2304 M^2(S+1) \log \left(2^{2L+5}N(L+1)d^2 (S+1)^{2L}\right) }{ N},
		\end{aligned}
	$$
	with the fixed point $r^*$ bounded by
	$$
		r^*\leq \dfrac{64M + 18432    M^2(S+1) \log \left(2^{2L+5}N(L+1)d^2 (S+1)^{2L}\right) }{N}.
	$$
	\label{lem:localrad}
\end{lemma}
\begin{proof}

	First, by applying Talagrand's contraction lemma with
	Assumption~\ref{ass:holder}, and the localized Dudley's theorem, we have
	\begin{equation}
		\begin{aligned}
			     & \fR_N(\{\ell\circ g | g\in \fN(L,\p,S,M), \E_{\widetilde \Pi}[\ell\circ g] \leq r\})                        \\
			\leq & 4M \fR_N\left(\left\{g - g^0 \big| g\in \fN(L,\p,S,M), \E_{\widetilde \Pi}[(g-g^0)^2] \leq r\right\}\right) \\
			\leq & 4M \E_{\Pi^{\otimes N}} \left[
			\inf_{\rho>0} \left(4\rho + 12 \int_\rho^{\sqrt{\hat{r}}}\sqrt{\dfrac{\log \N\left(u,\left\{g - g^0 \big| g\in \fN(L,\p,S,M),  \E_{\widetilde \Pi}[(g-g^0)^2] \leq r\right\},\|\cdot\|_{2,S}\right)}{N}}\d u\right)
			\right],
		\end{aligned}
		\label{eq:rad1}
	\end{equation}
	where the empirical localization radius depending on the sample set
	$S\sim\Pi^{\otimes N}$ is defined as
	$$\hat{r}:=\sup_{g\in\fN(L,\p,S,M),\E_{\widetilde
	\Pi}\left[(g-g^0)^2\right]\leq r}\|g-g^0\|_{2,S}^2,$$ and the last
	inequality is due to the fact that whenever $u>\sqrt{\hat{r}}$, the
	$u$-covering number with respect to $\|\cdot\|_{2,S}$ is exactly 1 and the
	integrand thus vanishes.

	By choosing $\rho = \frac{1}{N}$ on the RHS of~\eqref{eq:rad1}, we have
	$$
		\begin{aligned}
			     & \inf_{\rho>0}\left(4\rho + 12 \int_\rho^{\sqrt{\hat{r}}}\sqrt{\dfrac{\log \N\left(u,\left\{g - g^0 \big| g\in \fN(L,\p,S,M),  \E_{\widetilde \Pi}[(g-g^0)^2] \leq r\right\},\|\cdot\|_{2,S}\right)}{N}}\d u\right) \\
			\leq & \dfrac{4}{N} + 12 \int_{1/N}^{\sqrt{\hat{r}}}\sqrt{\dfrac{\log \N\left(u,\left\{g - g^0 \big| g\in \fN(L,\p,S,M), \E_{\widetilde \Pi}[(g-g^0)^2] \leq r\right\},\|\cdot\|_{2,S}\right)}{N}}\d u                   \\
			\leq & \dfrac{4}{N} + 12 \int_{1/N}^{\sqrt{\hat{r}}}\sqrt{\dfrac{\log \N\left(u,\left\{g - g^0 | g\in \fN(L,\p,S,M)\right\},\|\cdot\|_{2,S}\right)}{N}}\d u                                                            \\
			\leq & \dfrac{4}{N} + 12 \int_{1/N}^{\sqrt{\hat{r}}}\sqrt{\dfrac{\log \N\left(u,\fN(L,\p,S,M),\|\cdot\|_{\infty}\right)}{N}}\d u                                                                                       \\
			\leq    & \dfrac{4}{N} + 12 \int_{1/N}^{\sqrt{\hat{r}}}\sqrt{\dfrac{  (S+1) \log \left(2^{2L+5}u^{-1}(L+1)d^2 (S+1)^{2L}\right) }{N}}\d u                                                                                 \\
			\leq & \dfrac{4}{N} + 12 \sqrt{\hat{r}}\sqrt{\dfrac{  (S+1) \log \left(2^{2L+5}N(L+1)d^2 (S+1)^{2L}\right) }{N}},
		\end{aligned}
	$$
	where the third inequality is because of $\|\cdot\|_{2,S}\leq
	\|\cdot\|_\infty$, and the second to last equality is by
	Lemma~\ref{lem:covering}.

	Now set
	\begin{equation}
		\phi(r) = 4M\E_{\Pi^{\otimes N}}\left[\dfrac{4}{N} + 12 \sqrt{\hat{r}}\sqrt{\dfrac{  (S+1) \log \left(2^{2L+5}N(L+1)d^2 (S+1)^{2L}\right) }{N}}\right],
		\label{eq:phi}
	\end{equation}
	then $$\fR_N\left(\left\{\ell\circ g | g\in \fN(L,\p,S,M), \E_{\widetilde
	\Pi}[\ell \circ g] \leq r\right\}\right)\leq \phi(r)$$following the
	reasoning above.

	Now we turn to bound the empirical localization radius $\hat{r}$ again by
	the local Rademacher complexity. 
	$$
		\begin{aligned}
			\E_{\Pi^{\otimes N}}\left[\hat{r}\right] & =\E_{\Pi^{\otimes N}}\left[\sup_{g\in\fN(L,\p,S,M),\E_{\widetilde \Pi}\left[(g-g^0)^2\right]\leq r}\|g-g^0\|_{2,S}^2\right]                                                                                          \\
			                                         & =\E_{\Pi^{\otimes N}}\left[\sup_{g\in\fN(L,\p,S,M),\E_{\widetilde \Pi}\left[(g-g^0)^2\right]\leq r}\dfrac{1}{N}\sum_{i=1}^N\left(g(X_i) - g^0(X_i)\right)^2\right]                                                   \\
			                                         & =\E_{\Pi^{\otimes N}}\left[\sup_{g\in\fN(L,\p,S,M),\E_{\widetilde \Pi}\left[(g-g^0)^2\right]\leq r}\dfrac{1}{N}\sum_{i=1}^N\left(\left(g(X_i) - g^0(X_i)\right)^2 - \E_{\widetilde \Pi}\left[(g-g^0)^2\right]\right)\right] + r \\
			                                         & \leq 2\fR_N\left(\left\{(g - g^0)^2 | g\in \fN(L,\p,S,M), \E_{\widetilde \Pi}[(g-g^0)^2] \leq r\right\}\right)+ r\leq 2\phi(r)+r,
		\end{aligned}
	$$
	where the last inequality is by
	symmetrization~\cite{boucheron2013concentration}. 
    
    Then from~\eqref{eq:phi}, we have by Jensen's inequality,
	\begin{align*}
		\phi(r)\leq & 4M\left(\dfrac{4}{N}+12\sqrt{\E_{\Pi^{\otimes N}}\left[\hat{r}\right]}\sqrt{\dfrac{  (S+1) \log \left(2^{2L+5}N(L+1)d^2 (S+1)^{2L}\right) }{N}}\right) \\
		\leq        & 4M\left(\dfrac{4}{N}+12\sqrt{2\phi(r)+r}\sqrt{\dfrac{  (S+1) \log \left(2^{2L+5}N(L+1)d^2 (S+1)^{2L}\right) }{N}}\right)                                \\
		\leq        & 4M\left(\dfrac{4}{N}+12\left(\sqrt{r}+\sqrt{2\phi(r)}\right)\sqrt{\dfrac{  (S+1) \log \left(2^{2L+5}N(L+1)d^2 (S+1)^{2L}\right) }{N}}\right)            \\
		\leq        & \dfrac{16M}{N} + 48M\sqrt{\dfrac{ r(S+1) \log \left(2^{2L+5}N(L+1)d^2 (S+1)^{2L}\right) }{N}}                                                          \\
		            & + \dfrac{\phi(r)}{2} + \dfrac{ 2304 M^2(S+1) \log \left(2^{2L+5}N(L+1)d^2 (S+1)^{2L}\right) }{N},
	\end{align*}
	from which we deduce that
	$$
		\begin{aligned}
			     & \fR_N\left(\left\{\ell\circ g | g\in \fN(L,\p,S,M), \E_{\widetilde \Pi}[\ell\circ g] \leq r\right\}\right)               \\
			\leq & \phi(r)\leq \dfrac{32M}{N}+ 96 M\sqrt{\dfrac{ r(S+1) \log \left(2^{2L+5}N(L+1)d^2 (S+1)^{2L}\right) }{N}} \\
			     & \ \ \ \ \ \ \ \ \ \ + \dfrac{ 4608 M^2(S+1) \log \left(2^{2L+5}N(L+1)d^2 (S+1)^{2L}\right) }{ N},
		\end{aligned}
	$$
	which is clearly a sub-root function. Furthermore, by direct calculation,
	the fixed point $r^*$ of $\phi(r)$ satisfies
	$$
		r^*\leq \dfrac{64M + 18432 M^2(S+1) \log \left(2^{2L+5}N(L+1)d^2 (S+1)^{2L}\right) }{N},
	$$
	which concludes the proof.
\end{proof}

The following lemma describes the approximation capability of
$\fN(L,\p,S,\infty)$:
\begin{lemma}[Approximation error of $\fN(L,\p,S,\infty)$~{\cite[Theorem
	5]{schmidt2020nonparametric}}] For any $g\in \C^s(\Omega, M)$ and any
	integer $m\geq 1$ and $K\geq (s+1)^d\vee(M+1)e^d$, there exists
	$\bar{g}\in\fN(L,\p,S,\infty)$, where
	$$
		\begin{aligned}
			L  & = 8 + (m + 5)(1+\lceil \log_2(d\vee s) \rceil)                        \\
			\p & = \left( d,6(d+\lceil s\rceil)K,\cdots,6(d+\lceil s\rceil)K,1 \right) \\
			S  & \leq 141(d+s+1)^{3+r}K(m+6),
		\end{aligned}
	$$
	such that
	$$\|\bar{g}-g\|_{L^\infty(\Omega)}\leq (2M+1)(1+d^2+s^2)6^d K 2^{-m} + M 3^s
	K^{-s/d}.$$
	\label{lem:nnholder}
\end{lemma}

With all the above preparations, we are ready to prove
Theorem~\ref{thm:thetheorem}.
\begin{proof}[Proof of Theorem~\ref{thm:thetheorem}]
	Choose $m$ to be the smallest integer such that $K 2^{-m} \leq K^{-s/d}$,
	\emph{i.e.} $$ m = \left\lceil\frac{(1+s/d)\log K}{\log
	2}\right\rceil\lesssim \log K,$$ by Lemma~\ref{lem:nnholder}, for any $g\in
	\C^s(\Omega, M)$, there exists $\bar{g}\in\fN(L,\p,S,\infty)$, where
	\begin{equation}
		L \lesssim \log K,\quad \|\p\|_\infty\lesssim K,\quad S\lesssim K\log K,\
		\label{eq:lps}
	\end{equation}
	such that
	$$\|\bar{g}-g\|_{L^\infty(\Omega)}\lesssim K^{-s/d}.$$ Thus in
	Lemma~\ref{lem:localrad}, we have
	$$
		\begin{aligned}
			r^*\lesssim & \dfrac{M + M^2(K\log K+1) \left( (2\log K+5)\log 2 +\log N+ \log(\log K+1)+2\log d+ 2\log K\log(K\log K+1)\right) }{N} \\
			\lesssim    & \dfrac{K\log K (\log K\log (K\log K) + \log N)}{N} \lesssim \dfrac{K\log^3 K + K\log K\log N}{N}.
		\end{aligned}
	$$

	In Theorem~\ref{thm:hat}, we take $\delta' = N\upbeta(l\tau) = N^{-\alpha}$
	with an arbitrary $\alpha >0$, \emph{i.e.} $$l = \dfrac{\log
	M_{\widetilde{\Pi}}+(1+\alpha)\log N}{C_{\widetilde{\Pi}}\tau}\lesssim
	\dfrac{\log N}{\tau}.$$ Then with probability at least $1-2N^{-\alpha}$, we
	have
	$$
		\begin{aligned}
			\L_{\widetilde \Pi}^{g^0}(\hat g) \leq & \dfrac{1}{1-\eps} \hat{\L}_N^{g^0}(\hat{g}) + \dfrac{176}{M^2\eps} r^* + \dfrac{l\left(44\eps + 104 \right) M^2 \log\left(l N^\alpha\right)}{\eps N} \\
			\lesssim             & \hat{\L}_N^{g^0}(\hat{g}) + \dfrac{K\log^3 K + K\log K\log N}{N} + \dfrac{\log^2 N }{N\tau}.
		\end{aligned}
	$$

	In Lemma~\ref{lem:covering}, take $\rho = 1/N$ and $L$, $\p$, and $S$ as
	in~\eqref{eq:lps}, we have
	$$
		\log \N(\rho,\fN(L,\p,S,\infty),\|\cdot\|_\infty) \leq K\log^3 K + K\log K\log N,
	$$
	and thus in Theorem~\ref{thm:oracle} with also $\rho = 1/N$,
	$$
		\begin{aligned}
			\E\left[\hat{\L}_N^{g^0}(\hat{g})\right]\lesssim & \inf_{\hat{g}\in\fG}\E\left[ \hat{\L}_N^{g^0}(\hat{g})\right]+\tau + \dfrac{ \log \N(1/N,\fN(L,\p,S,M),\|\cdot\|_\infty)}{N\tau^\gamma}+\sqrt{\dfrac{ \rho^2}{N\tau^\gamma}} \\
			\lesssim                                  & \inf_{\hat{g}\in\fG}\|\hat{g}- g^0\|_\infty^2 + \tau + \dfrac{ \log \N(1/N,\fN(L,\p,S,\infty),\|\cdot\|_\infty)}{N\tau^\gamma}+\sqrt{\dfrac{1}{N^3\tau^\gamma}}        \\
			\lesssim                                  & K^{-2s/d} + \tau + \dfrac{ K\log^3 K + K\log K\log N}{N\tau^\gamma}+\dfrac{1}{N},
		\end{aligned}
	$$
	where we used $N\tau^\gamma\geq N\tau\geq 1$.

	Finally, we conclude
	$$
		\begin{aligned}
			\E\left[\L_{\widetilde \Pi}^{g^0}(\hat g)\right]\lesssim & \E\left[\hat{\L}_N^{g^0}(\hat{g})\right]+ \dfrac{K\log^3 K + K\log K\log N}{N} + \dfrac{\log^2 N }{N\tau}                                 \\
			\lesssim                               & K^{-2s/d} + \tau + \dfrac{  K\log^3 K + K\log K\log N}{N(\tau^\gamma\wedge 1)}+\dfrac{1}{N} + \dfrac{\log^2 N }{N\tau} \\
			\lesssim                               & \left(N(\tau^\gamma\wedge 1)\right)^{-\frac{2s}{2s+d}}\log^3 (N(\tau^\gamma\wedge 1))+\tau+ \dfrac{\log^2 N}{N\tau},
		\end{aligned}
	$$
	by taking $K\asymp (N(\tau^\gamma\wedge 1))^{\frac{d}{2s+d}}$.

\end{proof}

\section{Detailed Proofs of Section~\ref{sec:convergence}}
\label{sec:proof1}

In this section, we present the detailed proofs of the results in
Section~\ref{sec:convergence}, including Theorem~\ref{thm:drift} and
Theorem~\ref{thm:diffusion}. Both proofs are based on
Theorem~\ref{thm:thetheorem} that we have proved in Appendix~\ref{sec:proof2}.
Specifically, we would like to verify the conditions of
Theorem~\ref{thm:thetheorem} for the drift and diffusion estimation problems,
respectively. To this end, we first compute the noise $\Delta Z_{k\tau}$ of the
drift and diffusion estimators as in~\eqref{eq:estemploss}, analyze its
decomposition into the bias term $\Delta A_{k\tau}$ and the variance term
$\Delta M_{k\tau}$, and then verify Assumption~\ref{ass:bias} and find the
exponent $\gamma$ in Assumption~\ref{ass:mart} for both drift and diffusion
estimators.

\begin{lemma}
	For $s\in[k\tau,(k+1)\tau]$, $$\E\left[\|\x_t-\x_{k\tau}\|^2\mid
	\F_0^{k\tau}\right]\leq e^\tau(M^2 d + 2M d^2)(t-k\tau).$$
\label{lem:x2}
\end{lemma}
\begin{proof}
    By It\^o's formula, we have for $s\in[k\tau,(k+1)\tau]$,
    $$
        \begin{aligned}
            \d \|\x_s-\x_{k\tau}\|^2 & =2(\x_s-\x_{k\tau})^\top \d \x_s+d \d \langle \x_s-\x_{k\tau},\x_s-\x_{k\tau}\rangle\\
            & =2(\x_s-\x_{k\tau})^\top \left(\b_s\d s+\Sigma_s\d \w_s\right)+ d \d \langle \x_s-\x_{k\tau},\x_s-\x_{k\tau}\rangle \\
            & =2(\x_s-\x_{k\tau})^\top \left(\b_s\d s+\Sigma_s\d \w_s\right)+d\tr\left(\Sigma_s^\top \Sigma_s\right)\d s     \\
            & =\left(2(\x_s-\x_{k\tau})^\top \b_s + 2d\tr D_s\right)\d s+2(\x_s-\x_{k\tau})^\top \Sigma_s\d \w_s.
        \end{aligned}
    $$
    Therefore for any $t\in[k\tau,(k+1)\tau]$,
    $$
        \begin{aligned}
            \E\left[\|\x_t-\x_{k\tau}\|^2\mid \F_0^{k\tau}\right] & =\E\left[\int_{k\tau}^t\left(2(\x_s-\x_{k\tau})^\top \b_s + 2d\tr D_s\right)\d s \bigg| \F_0^{k\tau}\right] \\
            & \leq ( M^2 d + 2M d)(t-k\tau) + \int_{k\tau}^t\E\left[\|\x_s-\x_{k\tau}\|^2\mid \F_0^{k\tau}\right]\d s,
        \end{aligned}
    $$
    and by Gronwall's inequality, we have
    $$
        \E\left[\|\x_t-\x_{k\tau}\|^2\mid \F_0^{k\tau}\right]\leq e^\tau(M^2 d + 2M d^2)(t-k\tau).
    $$
\end{proof}

\subsection{Proof of Theorem~\ref{thm:drift}}
\label{sec:proofdrif}

As a warm-up for the proof of Theorem~\ref{thm:diffusion}, we first prove the
corresponding upper bound for drift estimation:
\begin{proof}[Proof of Theorem~\ref{thm:drift}]
    Suppose that $1\leq i\leq d$. Let $b^i_t$, the $i$-th component of the drift
    $\b_t$, be our target function $g^0$. We also use $Z^i_t$, $A^i_t$, and
    $M^i_t$ to denote the corresponding noise terms.

    For any $0\leq k\leq N-1$, we have by It\^o's formula,
	$$
		\dfrac{x^i_{(k+1)\tau}-x^i_{k\tau}}{\tau} = b^i_{k\tau}+\dfrac{1}{\tau}\int_{k\tau}^{(k+1)\tau}\left(b^i_t-b^i_{k\tau}\right)\d t+\dfrac{1}{\tau}\int_{k\tau}^{(k+1)\tau}\bSigma_t^i\d \w_t,
	$$
    and compare the estimated empirical loss for drift
    inference~\eqref{eq:driftemploss} with the general
    form~\eqref{eq:estemploss}, we have
    $$
    \Delta Z^i_{k\tau} = \dfrac{1}{\tau}\int_{k\tau}^{(k+1)\tau}\left(b^i_t-b^i_{k\tau}\right)\d t+\dfrac{1}{\tau}\int_{k\tau}^{(k+1)\tau}\bSigma_t^i\d \w_t.
    $$
    By definition, we also have
	$$
        \Delta A^i_{k\tau}=\E\left[\dfrac{1}{\tau}\int_{k\tau}^{(k+1)\tau}\left(b^i_t-b^i_{k\tau}\right)\d t\right],\quad
		\Delta M^i_{k\tau}=\dfrac{1}{\tau}\int_{k\tau}^{(k+1)\tau}\bSigma_t^i\d \w_t.
	$$

    The following simple argument
	$$
		\Delta \qv{M^i}_{k\tau}=\dfrac{1}{\tau^2}\int_{k\tau}^{(k+1)\tau}\|\bSigma_t^i\|_2^2 \d t \leq M^2d \tau^{-1}
	$$
    validates Assumption~\ref{ass:mart} with $\gamma = 1$.
    
    Also, by Cauchy-Schwarz and Lemma~\ref{lem:x2}, we have
	$$\begin{aligned} \E\left[ \left( \Delta A^i_{k\tau}\right)^2 \big|
			\F_0^{k\tau} \right]\leq &
			\E\left[\dfrac{1}{\tau^2}\left(\int_{k\tau}^{(k+1)\tau}\left(b^i_t-b^i_{k\tau}\right)\d
			t\right)^2 \bigg| \F_0^{k\tau}\right] \\
			\leq                                                     &
			\E\left[\dfrac{1}{\tau}\int_{k\tau}^{(k+1)\tau}\left(b^i_t-b^i_{k\tau}\right)^2\d
			t\bigg| \F_0^{k\tau}\right]                \\
			\leq                                               &
			\E\left[\dfrac{M^2}{\tau}\int_{k\tau}^{(k+1)\tau}\left\|\x_t-\x_{k\tau}\right\|^2\d
			t\bigg| \F_0^{k\tau}\right]                  \\
			\leq                                                 &
			\dfrac{1}{\tau}\int_{k\tau}^{(k+1)\tau}e^\tau(M^2 d + 2M
			d^2)(t-k\tau)\d t \\
            \leq & \dfrac{1}{2}e^\tau(M^2 d + 2M d^2)\tau,
		\end{aligned}
	$$
	and Assumption~\ref{ass:bias} is thus satisfied.
    
    Finally, we apply Theorem~\ref{thm:thetheorem} to $g^0 = b^i$, and obtain the final rate
    $$
    \begin{aligned}
        \E\left[\hat{\L}_N^{g^0}(\hat{b^i})\right]\lesssim&
        \left(N\tau\right)^{-\frac{2s}{2s+d}}\log^3 (N\tau)+\tau+ \dfrac{\log^2
        N}{N\tau}\\
        \lesssim& T^{-\frac{2s}{2s+d}}\log^3 T+\tau,
    \end{aligned}
    $$
    and the proof is thus complete by repeating for all $1\leq i\leq d$.
\end{proof}

\subsection{Proof of Theorem~\ref{thm:diffusion}}
\label{sec:proofdiff}

The proof for diffusion estimation requires a even more delicate analysis on the dynamics of the SDE~\eqref{eq:sde}.
\begin{proof}[Proof of Theorem~\ref{thm:diffusion}]
    Following the notations introduced in the proof of Theorem~\ref{thm:diffusion}, for $1\leq i,j\leq d$, let $D^{ij}$ be the target function and $Z^{ij}$, $A^{ij}$, $M^{ij}$ be the corresponding noise terms.

	For any $0\leq k\leq N-1$, define the following auxiliary process for $s\in[0,\tau]$,
    \begin{equation}
        Y_{k,s}^i :=x^i_{k\tau+s}-x^i_{k\tau}-\hat{b}^i_{k\tau}s  
             = \int_{k\tau}^{k\tau+s}\left(b^i_t-\hat{b}^i_{k\tau}\right)\d t+\int_{k\tau}^{k\tau+s}\bSigma_t^i\d \w_t ,
             \label{eq:Y}
    \end{equation}
    where $\bSigma^i$ denotes the $i$-th row of $\Sigma$ (a row vector).
	Plugging~\eqref{eq:Y} into the estimated empirical loss for diffusion
	estimation~\eqref{eq:diffusionemploss}, we have by definition
    $$
    \begin{aligned}
        \Delta Z^{ij}_{k\tau} &= (2\tau)^{-1}
        Y_{k,\tau}^iY_{k,\tau}^j-D^{ij}_{k\tau}\\
    \Delta
    A^{ij}_{k\tau}&=(2\tau)^{-1}\E\left[Y_{k,\tau}^iY_{k,\tau}^j\right]-D^{ij}_{k\tau}\\
    \Delta \qv{M^{ij}}_{k\tau}& = (2\tau)^{-2}\qv{Y_{k}^iY_{k}^j}_\tau.
    \end{aligned}
    $$

    The process $Y_{k,s}^i$ satisfies the following SDE:
    $$\d Y_{k,s}^i = \left(b_{k\tau+s}^i-\hat{b}_{k\tau}^i\right)\d s+ \bSigma_{k\tau+s}^i \d \w_s.$$
	By It\^{o}'s formula, the process $\left(Y_{k,s}^i Y_{k,s}^j\right)_{0\leq
	s\leq \tau}$ as the product of $Y_{k,s}^i$ and $Y_{k,s}^j$ also satisfies an
	SDE as follows:
	$$
		\begin{aligned}
			&\d \left(Y_{k,s}^iY_{k,s}^j\right) =  Y_{k,s}^i \d Y_{k,s}^j + Y_{k,s}^j \d Y_{k,s}^i + \d \left\langle Y_{k,s}^i, Y_{k,s}^j\right\rangle                                                                                                                                       \\
			=                         & \left[Y_{k,s}^i \left(b_{k\tau+s}^j-\hat{b}_{k\tau}^j\right)+Y_{k,s}^j \left(b_{k\tau+s}^i-\hat{b}_{k\tau}^i\right)+\bSigma^i_{k\tau+s}\left(\bSigma^j_{k\tau+s}\right)^\top\right]\d s+\left[Y_{k,s}^i\bSigma_{k\tau+s}^j+Y_{k,s}^j\bSigma_{k\tau+s}^i\right]\d\w_s
		\end{aligned}
	$$
	with its quadratic variation satisfying
	$$\d \qv{Y_{k}^iY_{k}^j}_s =
	\left|Y_{k,s}^i\bSigma_{k\tau+s}^j+Y_{k,s}^j\bSigma_{k\tau+s}^i \right|^2\d
	s.$$

	By direct calculation, we have 
	\begin{equation}
		\begin{aligned}
			  & \E\left[Y_{k,s}^iY_{k,s}^j \big| \F_0^{k\tau}\right]                                                                                                                                                                                                 \\
			= & \E\left[\int_0^{s}\left(\int_0^{s_1} \left(b_{k\tau+s_2}^i-\hat{b}_{k\tau}^i\right)\d s_2+ \int_0^{s_1}\bSigma_{k\tau+s_2}^i \d \w_{s_2}\right) \left(b_{k\tau+s_1}^j-\hat{b}_{k\tau}^j\right) \d s_1\bigg| \F_0^{k\tau}\right]                                        \\
			+ & \E\left[\int_0^{s} \left(\int_0^{s_2} \left(b_{k\tau+s_1}^j-\hat{b}_{k\tau}^j\right)\d s_1+ \int_0^{s_2}\bSigma_{k\tau+s_1}^j \d \w_{s_1}\right) \left(b_{k\tau+s_2}^i-\hat{b}_{k\tau}^i\right) \d s_2\bigg| \F_0^{k\tau}\right]                                      \\
			+  & \E\left[\int_0^s \bSigma^i_{k\tau+s'}\left(\bSigma^j_{k\tau+s'}-\bSigma_{k\tau}^j\right)^\top + \left(\bSigma^i_{k\tau+s'}-\bSigma^i_{k\tau}\right)\left(\bSigma_{k\tau}^j\right)^\top \d s'\bigg| \F_0^{k\tau}\right]+s\bSigma_{k\tau}^i\left(\bSigma_{k\tau}^{j}\right)^\top \\
			= & \int_0^s\int_0^s \E\left[\left(b_{k\tau+s_1}^i-\hat{b}_{k\tau}^i\right)\left(b_{k\tau+s_2}^j-\hat{b}_{k\tau}^j\right)\bigg| \F_0^{k\tau}\right]\d s_1\d s_2                                                                                                      \\
			  & +\E\left[\int_0^s \bSigma^i_{k\tau+s'}\left(\bSigma^j_{k\tau+s'}-\bSigma_{k\tau}^j\right)^\top + \left(\bSigma^i_{k\tau+s'}-\bSigma^i_{k\tau}\right)\left(\bSigma_{k\tau}^j\right)^\top \d s'\bigg| \F_0^{k\tau}\right]+2sD^{ij}_{k\tau}.
		\end{aligned}
                \label{eq:yiyj}  
	\end{equation}

	For the first term of~\eqref{eq:yiyj}, we have the following bound by
	applying Cauchy-Schwarz 
	$$
		\begin{aligned}
			         & \left(\int_0^s\int_0^s \E\left[\left(b_{k\tau+s_1}^i-\hat{b}_{k\tau}^i\right)\left(b_{k\tau+s_2}^j-\hat{b}_{k\tau}^j\right)\bigg| \F_0^{k\tau}\right]\d s_1\d s_2\right)^2                                                                    \\
			\lesssim & \left(\int_0^s\int_0^s \E\left[\left(b_{k\tau+s_1}^i-b_{k\tau}^i\right)\left(b_{k\tau+s_2}^j-b_{k\tau}^j\right)\bigg| \F_0^{k\tau}\right]\d s_1\d s_2\right)^2                                                                                \\
			         +&\left(\int_0^s\int_0^s\E\left[\left(b_{k\tau}^i-\hat{b}_{k\tau}^i\right)\left(b_{k\tau+s_2}^j-b_{k\tau}^j\right)\bigg| \F_0^{k\tau}\right]\d s_1\d s_2\right)^2                                                                        \\+ & \left(\int_0^s\int_0^s\E\left[\left(b_{k\tau+s_1}^i-b_{k\tau}^i\right)\left(b_{k\tau}^j-\hat{b}_{k\tau}^j\right)\bigg| \F_0^{k\tau}\right]\d s_1\d s_2\right)^2\\
                     + & \left(\int_0^s\int_0^s\E\left[\left(b_{k\tau}^i-\hat{b}_{k\tau}^i\right)\left(b_{k\tau}^j-\hat{b}_{k\tau}^j\right)\bigg| \F_0^{k\tau}\right]\d s_1\d s_2\right)^2\\
    \lesssim &s^2\int_0^s\int_0^s
    \E\left[\|\x_{k\tau+s_1} - \x_{k\tau}\|^2\|\x_{k\tau+s_2} - \x_{k\tau}\|^2|\F_0^{k\tau}\right]
    \d s_1\d s_2\\
    + &   s^3\left(b_{k\tau}^i-\hat{b}_{k\tau}^i\right)^2 \int_0^s
    \E\left[ \|\x_{k\tau+s_1} - \x_{k\tau}\|^2|\F_0^{k\tau}\right]
    \d s_1 \\
    +&  s^3\left(b_{k\tau}^j-\hat{b}_{k\tau}^j\right)^2 \int_0^s
    \E\left[\|\x_{k\tau+s_2} - \x_{k\tau}\|^2|\F_0^{k\tau}\right]
    \d s_2 +s^4\left(b_{k\tau}^i-\hat{b}_{k\tau}^i\right)^2\left(b_{k\tau}^j-\hat{b}_{k\tau}^j\right)^2 \\
            \lesssim& s^6 + s^5 \left[\left(b_{k\tau}^i-\hat{b}_{k\tau}^i\right)^2 +\left(b_{k\tau}^j-\hat{b}_{k\tau}^j\right)^2 \right]+ s^4\left(b_{k\tau}^i-\hat{b}_{k\tau}^i\right)^2\left(b_{k\tau}^j-\hat{b}_{k\tau}^j\right)^2,
\end{aligned}
    $$
    where the last inequality follows from Lemma~\ref{lem:x2}.

	For the second term of~\eqref{eq:yiyj},
	\begin{equation}
        \begin{aligned}
	        & \left(\E\left[\int_0^s \bSigma^i_{k\tau+s'}\left(\bSigma^j_{k\tau+s'}-\bSigma_{k\tau}^j\right)^\top + \left(\bSigma^i_{k\tau+s'}-\bSigma^i_{k\tau}\right)\left(\bSigma_{k\tau}^j\right)^\top \d s'\bigg|\F_0^{k\tau}\right]\right)^2 \\
			\leq     & \E\left[\left(\int_0^s \bSigma^i_{k\tau+s'}\left(\bSigma^j_{k\tau+s'}-\bSigma_{k\tau}^j\right)^\top + \left(\bSigma^i_{k\tau+s'}-\bSigma^i_{k\tau}\right)\left(\bSigma_{k\tau}^j\right)^\top \d s'\right)^2\bigg|\F_0^{k\tau}\right] \\
			\lesssim & s \int_0^s \E\left[\left\|\bSigma^j_{k\tau+s'}-\bSigma^j_{k\tau}\right\|^2\bigg|\F_0^{k\tau}\right]\d s'      +s \int_0^s \E\left[\left\|\bSigma^i_{k\tau+s'}-\bSigma^i_{k\tau}\right\|^2\bigg|\F_0^{k\tau}\right]\d s'                                                                                                              \\
			\lesssim & s \int_0^s \E\left[\left\|\x_{k\tau+s'}-\x_{k\tau}\right\|^2\right]\d s' = s^3.
		\end{aligned}
        \label{eq:secondterm}
    \end{equation}
	Combining the two estimations above, we have
    \begin{equation}
        \begin{aligned}
            &\left(\E\left[\left(Y_{k,s}^iY_{k,s}^j \right) \big| \F_0^{k\tau}\right]\right)^2\\
           \lesssim& s^6+ s^5 \left[\left(b_{k\tau}^i-\hat{b}_{k\tau}^i\right)^2 +\left(b_{k\tau}^j-\hat{b}_{k\tau}^j\right)^2 \right]+s^4\left(b_{k\tau}^i-\hat{b}_{k\tau}^i\right)^2\left(b_{k\tau}^j-\hat{b}_{k\tau}^j\right)^2+s^3 + s^2 \left(D_k^{ij}\right)^2\lesssim s^2,
         \end{aligned}\
         \label{eq:yiyj2}
    \end{equation}
    and thus
	$$
		\begin{aligned}
            &\E\left[\dfrac{1}{N}\sum_{k=0}^{N-1} (\Delta A_{k\tau})^2 \right]=\E\left[\dfrac{1}{N}\sum_{k=0}^{N-1}\left((2\tau)^{-1}\E\left[Y_{k,\tau}^iY_{k,\tau}^j\big|\F_0^{k\tau}\right]-D^{ij}_{k\tau}\right)^2 \right] \\
        \lesssim& \tau^{-2}\left(\tau^6+\tau^5 \left[\left(b_{k\tau}^i-\hat{b}_{k\tau}^i\right)^2 +\left(b_{k\tau}^j-\hat{b}_{k\tau}^j\right)^2 \right]+\tau^4\left(b_{k\tau}^i-\hat{b}_{k\tau}^i\right)^2\left(b_{k\tau}^j-\hat{b}_{k\tau}^j\right)^2+\tau^3\right)   \lesssim \tau,
        \end{aligned}
	$$
    confirming Assumption~\ref{ass:bias} for diffusion estimation.
	
	Moreover, Assumption~\ref{ass:mart} is also satisfied with $\gamma = 0$ by
	checking
	$$
		\begin{aligned}
			\E\left[\Delta \qv{M^{ij}}_{k\tau}\big|\F_0^{k\tau}\right] & = \E\left[(2\tau)^{-2}\qv{Y_{k}^iY_{k}^j}_\tau\big|\F_0^{k\tau}\right] \\
            &= (2\tau)^{-2}\int_0^\tau \E\left[\left|Y_{k,s}^i\bSigma_s^j+Y_{k,s}^j\bSigma_s^i \right|^2\big|\F_0^{k\tau}\right]\d s \\
		& \lesssim \tau^{-2} \int_0^\tau \E\left[\left(Y_{k,s}^i \right)^2 + \left(Y_{k,s}^j \right)^2\big|\F_0^{k\tau}\right]\d s \\
        &\lesssim \tau^{-2}\sqrt{\tau\int_0^\tau\left( \E\left[(Y_{k,s}^i)^2\right]\right)^2\d s}\\
        &\lesssim \tau^{-2}\sqrt{\tau\int_0^\tau s^2 \d s}
        \lesssim 1,
		\end{aligned}
	$$
    where the second to last inequality is by taking $i=j$ in~\eqref{eq:yiyj2}.

    Finally, we apply Theorem~\ref{thm:thetheorem} to $g^0 = D^{ij}$, and obtain
    the final rate
    $$
        \E\left[\hat{\L}_N^{g^0}(\hat{D^{ij}})\right]\lesssim
        N^{-\frac{2s}{2s+d}}\log^3N+\tau+ \dfrac{\log^2 N}{N\tau}
        \lesssim N^{-\frac{2s}{2s+d}}\log^3 N+\tau + \dfrac{\log^2 N}{T}.
    $$
    The proof is thus complete by repeating for all $1\leq i,j\leq d$.

\end{proof}

\begin{remark}
    In Algorithm~\ref{alg:alg}, we used the estimator $\hat{\b}$ obtained in the
    first step as an approximation of the true drift $\b$ in
    $\tilde{\L}^D_N(\bar{D};(\x_{k\tau})_{k=0}^N, \hat \b)$. It turns out that
    this approximation does not affect the overall convergence rate, since we
    concluded that the leading bias term is~\eqref{eq:secondterm}, which is
    caused by the finite resolution of the observations. However, an accurate
    estiamtor $\hat{\b}$ would reduce the variance indeed and thus improve the
    performance of the neural diffusion estiamtor.
\end{remark}

\end{document}